\documentclass[11pt]{article}

\usepackage{bbm}
\usepackage{amsmath,amsfonts,amssymb,enumerate}
\usepackage{color}
\usepackage{lscape}
\usepackage{latexsym}
\usepackage{multirow}
\usepackage{rotating}
\usepackage{threeparttable}
\usepackage{graphicx}
\usepackage{algpseudocode,algorithm}
\usepackage[margin=1in]{geometry}

\numberwithin{equation}{section}
\newtheorem{theorem}{\bf Theorem}[section]

\newtheorem{corollary}[theorem]{\bf Corollary}

\def \bR {\mathbb{R}}
\def \bN {\mathbb{N}}
 
\def\DD{\mathbb{D}} 

\def\cC{{\cal{C}}}

\def\cG{{\cal{G}}}
\def\cH{{\cal{H}}}

\def\cN{{\cal{N}}}

\def\bbb{{\bf{b}}}

\def\ff{{\bf{f}}}
\def\bg{{\bf{g}}}

\def\be{{\bf{e}}}
\def\bx{{\bf{x}}}
\def\by{{\bf{y}}}

\def\b0{{\bf{0}}}

\def\bA{{\bf{A}}}

\def\bW{{\bf{W}}}

\newenvironment{proof}{\noindent{\em Proof:}}{\quad \hfill$\Box$\vspace{2ex}}

\graphicspath{{figs/}}

\begin{document}



\title{\bf Multi-Grade Deep Learning}
\author{Yuesheng Xu\thanks{Department of Mathematics and Statistics, Old Dominion University, Norfolk, VA 23529, USA. E-mail address: {\it y1xu@odu.edu}. }}
\date{}
\maketitle

\maketitle
\begin{abstract}
Deep learning requires solving a nonconvex optimization problem of a large size to learn a deep neural network. The current deep learning model is of a {\it single-grade}, that is, it learns a deep neural network by solving a single nonconvex optimization problem. When the layer number of the neural network is large, it is computationally challenging to carry out such a task efficiently. The complexity of the task comes from learning all weight matrices and bias vectors from one single nonconvex optimization problem of a large size. 
Inspired by the human education process which arranges learning in grades, we propose a multi-grade learning model: Instead of solving one single optimization problem of a large size, we successively solve a number of optimization problems of small sizes, which are organized in grades, to learn a shallow neural network for each grade. Specifically, the current grade is to learn the leftover from the previous grade. In each of the grades, we learn a shallow neural network stacked on the top of the neural network, learned in the previous grades, which remains unchanged in training of the current and future grades. By dividing the task of learning a deep neural network into learning several shallow neural networks, one can alleviate the severity of the nonconvexity of the original optimization problem of a large size. When all grades of the learning are completed, the final neural network learned is  a {\it stair-shape} neural network, which is the {\it superposition} of networks learned from all grades. Such a model enables us to learn a deep neural network much more effectively and efficiently.
Moreover, multi-grade learning naturally leads to adaptive learning.
We prove that in the context of function approximation if the neural network generated by a new grade is nontrivial, the optimal error of the grade is strictly reduced from the optimal error of the previous grade. Furthermore, we provide several proof-of-concept numerical examples which demonstrate that the proposed multi-grade model outperforms significantly the traditional single-grade model and is much more robust than the traditional model.
\end{abstract}


\section{Introduction}
The immense success of deep learning (LeCun et al., 2015; Goodfellow et al., 2016) has been widely recognized. Its great impact to science and technology has been witnessed (H\"aggstr\"om et al., 2019;  Krizhevsky et al., 2012; Shen et al., 2017; Torlai et al., 2018) and will continue to deepen. From a mathematical perspective, such successes are mainly due to the powerful expressiveness of deep neural networks in representing a function (Daubechies et al., 2022; Shen et al., 2021). The deep neural networks are learned by solving optimization problems which determine their parameters (weight matrices and bias vectors) that define them with activation functions. The optimization problems are highly nonconvex and have large numbers of parameters. Solving such optimization problems is challenging. Due to having large numbers of parameters, finding a global minimizer is difficult. By employing the stochastic gradient descent method (Bottou 1998; Bottou et al., 2012; Kingma and Ba, 2015) to solve the optimization problems, most likely only local minimizers may be found, since gradient-based optimization starting from random initialization appears to often get stuck in poor solutions (Bengio et al., 2007). Moreover, convergence of the iteration is very slow. 
This has been recognized as a major computational obstacle of deep learning.

The current model of deep learning uses a {\it single} optimization problem to train all parameters needed for the desired deep neural network. We will refer it to as the single-grade model. When solving a single-grade learning model, the more layers the neural network possesses, the severer its nonconvexity is, and thus, more difficulty we would encounter when trying to solve it. The single-grade deep learning model is like to teach a middle school student to learn Calculus, who has no required preliminary knowledge such as elementary algebra, trigonometry or college algebra. Even if it is not an impossible mission, it is extremely difficult.


The goal of this study is to address this challenge. Inspired by human learning process, we propose multi-grade deep learning models in which deep neural networks are learned grade-by-grade from shallow to depth, with redundancy. Instead of solving one single optimization problem with a big number of parameters, we will solve several optimization problems, each with a significantly smaller number of parameters which determine a shallow neural network. After several grades of learning, we build a deep neural network with a structure different from the one learned by the single-grade learning but with a comparable approximation accuracy (or even better). The neural network learned from the proposed multi-grade learning model is the superposition of {\it stairs-shape} neural networks. Since often it is easier to learn several shallow neural networks than a deep one, such models enable us to learn deep neural networks much more effectively and efficiently, avoiding solving one single optimization problem with a large number of parameters. 

The proposed multi-grade deep learning model imitates the human learning process. In our modern society, human education is often organized in five large stages, which include elementary school, middle school, high school, college and graduate school. Each of these stages of schooling is further divided into different grades. For example, elementary school is normally divided into five grades and its curriculums are designed according to the grades, with substantial amount of redundancy, proceeding from the surface to the inner essence, from an easy level to a sophisticated level. The knowledge that students had learned in a current grade would serve as a basis to learn new knowledge in the next grade. At the end of each grade, there are examinations to ensure that students do learn what are required to know. The human learning experience shows that the multi-grade learning is effective and efficient.



At the technical level, the development of multi-grade learning models is influenced by the multilevel augmentation method for solving nonlinear system of a large size which results from discretization of Hammerstein equations, a class of nonlinear integral equations, (Chen et al., 2009), see also (Chen et al., 2015).  Solving a Hammerstein equation by projecting it onto a piecewise polynomial subspace that has a multiscale structure including higher levels of resolutions boils down to solving a discrete nonlinear system of a large size. The more levels of resolutions are included in the subspace, the higher the accuracy of the approximate solution has. However,
solving a nonlinear system of a large size is computational expansive. With the help of the multiscale analysis (Chen et al., 1999; Daubechies, 1992; Micchelli and Xu, 1994), we reformulated the process of solving the nonlinear system of a large size into two major steps: First, we solve a nonlinear system which result from projecting the Hammerstein equation onto a subspace that contains only lower resolution levels. We then compensate the error by solving a linear system corresponding to the subspace that contains high resolution levels, and this two-step process is repeated. With this reformulation, we avoid solving a nonlinear system of a large size, and instead, we solve a nonlinear system of a small size several times with subsidizing by solving linear systems of large sizes. It was proved in (Chen et al., 2009) that this method generates an approximate solution having the same accuracy order as solving the original nonlinear system of a large size, but with significantly less computational costs.  Inspired by such an idea, we propose in this paper to reformulate the optimization problem for learning the entire neural network as many optimization problems for learning parts of a deep neural network. These optimization problems are arranged in grades in the way that learning of the current grade is to abstract information from the {\it remainder} of learning of the previous grade. In each of the grades, we solve an optimization problem of a significantly smaller size, which is much easier to solve than the original entire problem. The present grade is to learn from the leftover of the previous grade. In fact, according to (Wu et al., 2005; Wu and Xu, 2002) under certain conditions, online gradient methods for training a class of nonlinear feedforward shallow neural networks has guaranteed deterministic convergence. The multi-grade learning model takes this advantage by training a number of {\it shallow} neural networks which is a substantially easier task, avoiding training a {\it deep} neural network which is a much difficult task.

%

A deep neural network learned by a multi-grade learning model  differs significantly from deep neural networks learned by the existing learning models. In a multi-grade learning model, each grade updates the knowledge gained from learning of the previous grades and the total knowledge learned up to the current grade is the accumulation of the knowledge learned in {\it all} previous grades plus the update in the current grade. Mathematically, the deep neural network learned up to the current grade is the superposition of all the neural networks learned in all grades so far, (see, Figure \ref{TwoGradeLearningModel} for an illustration). Such a neural network has great redundancy which enhances the expressiveness in presenting data/functions.

The multi-grade learning is suitable for adaptive computation. It naturally allows us to add a new shallow neural network to the neural network learned from the previous grades without changing it. In particular, the multi-grade learning model is natural for learning a solution of an operator equation such as a partial differential equation and an integral equation, especially when its solution has certain singularity, since it has a innate posterior error.

We organize this paper in seven sections. In section 2, we review the single-grade learning model.
Section 3 is devoted to the development of multi-grade models. We then establish in section 4 theoretical results that justify the proposed model. In section 5, we discuss crucial issues related to the implementation of the proposed multi-grade learning models. In section 6, we provide several proof-of-concept numerical examples. Finally, we make conclusive remarks in section 7.

\section{Deep Neural Networks: Single-Grade Learning}

In this section, we recall the definition of the standard single-grade deep learning model, and discuss its computational challenges.

Many real world problems require to learn a function from given data. Let $s, t\in \bN$. Suppose that we are given $m$ pairs of points $(\bx_i, \by_i)$, $i\in\bN_m:=\{1,2,\dots,m\}$, where $\bx_i\in\bR^s$ and $\by_i\in \bR^t$. We wish to learn from this set of data a function $\ff:\bR^s\to\bR^t$, which represents intrinsic information embedded in the data. Deep neural networks furnish us with an excellent representation of the function due to their gifted expressiveness as a result of their special structure.

We first recall the definition of a deep neural network.
A deep neural network of $n$ layers composed of $n-1$ hidden layers and one output layer is constructed with $n$ weight matrices $\bW_k$ and bias vectors $\bbb_k$, $k\in\bN_{n}$, through a pre-selected activation function $\sigma: \bR\to\bR$. Specifically, we suppose that $d\in\bN$, and for a vector $\bx:=[x_1, x_2,\dots, x_d]^\top\in\bR^d$, we define
the vector-valued function by the activation function
\begin{equation}\label{activationF}
\sigma(\bx):=[\sigma(x_1),\dots,\sigma(x_d)]^\top.
\end{equation}
As in (Xu and Zhang, 2021), for $n$ vector-valued functions $f_k$, $k\in\bN_n$, such that the range of $f_k$ is contained in the domain of $f_{k+1}$, $k\in\bN_{n-1}$, the consecutive composition of $f_k$, $k\in\bN_n$, is denoted by
\begin{equation}\label{consecutive_composition}
    \bigodot_{k=1}^n f_k:=f_n\circ f_{n-1}\circ\cdots\circ f_2\circ f_1,
\end{equation}
whose domain is that of $f_1$.
For  $\bW_i\in\bR^{m_i\times m_{i-1}}$,  with $m_0:=s$, $m_n:=t$ and $\bbb_i\in\bR^{m_i}$, $i\in\bN_n$, 
a deep neural network is a function defined by
\begin{equation}\label{DNN}
    \cN_n(\bx):=\left(\bW_n\bigodot_{i=1}^{n-1} \sigma(\bW_i \cdot+\bbb_i)+\bbb_n\right)(\bx),\ \ \bx\in\bR^s.
\end{equation}
The $n$-th layer is the output layer.
Clearly, $\cN_n: \bR^s\to\bR^t$ is a vector-valued function.
From \eqref{DNN} and the definition \eqref{activationF}, we have the recursion
\begin{equation}\label{Step1}
    \cN_1(\bx):=\sigma(\bW_1 \bx+\bbb_1)
\end{equation}
and
\begin{equation}\label{Recursion}
    \cN_{k+1}(\bx)=\sigma(\bW_{k+1}\cN_k(\bx)+\bbb_{k+1}), \ \ \bx\in \bR^s, \ \ \mbox{for all} \ \ k\in \bN_{n-1}.
\end{equation}
Note that when $k:=n-1$, $\sigma$ in \eqref{Recursion} is viewed as the identity map.

We now return to learning a function from a given data set.
For $m$ pairs of given data points  $(\bx_i, \by_i)$, $i\in\bN_m$,
one can learn a function 
\begin{equation}\label{TraditionalDNN}
  \cN_n(\bx):= \cN_n(\{\bW_j^*,\bbb_j^*\}_{j=1}^n;\bx),\ \ \bx\in\bR^s
\end{equation}
by solving the parameters $\{\bW_j^*, \bbb_j^*\}_{j=1}^n$ with $\bW_j^*\in\bR^{m_j\times m_{j-1}}$ and $\bbb_j^*\in\bR^{m_j}$ from the minimization problem
\begin{equation}\label{Basic-Min-Problem}
     \min\left\{\sum_{k=1}^m\|\cN_n(\{\bW_j,\bbb_j\}_{j=1}^n; \bx_k)-\by_k\|_{\ell_2}^2: \bW_j\in\bR^{m_j\times m_{j-1}}, \bbb_j\in\bR^{m_j}, j\in\bN_{n}\right\},
\end{equation}
where $\|\cdot\|_{\ell_2}$ denotes the Euclidean vector norm of $\bR^t$. When we actually solve  problem \eqref{Basic-Min-Problem} we may need to add an appropriate regularization term if 
overfitting occurs. We postpone this issue until later so that we can concentrate on crucial conceptual issues.

Learning problem \eqref{Basic-Min-Problem} has a continuous counterpart in the context of function approximation. 
For a vector-valued function $\bg:=[g_1,\dots, g_t]^\top: \DD\subseteq\bR^s\to\bR^t$, we define 
$$
\|\bg\|:=\left[\sum_{j=1}^t \|g_j\|_2^2\right]^\frac{1}{2},
\ \
\mbox{where}\ \
\|g_j\|_2:=\left(\int_\DD|g_j(\bx)|^2d\bx\right)^\frac12.
$$
By $L_2(\DD,\bR^t)$ we denote the Hilbert space of the functions $\bg:\DD\to\bR^t$ with $\|\bg\|<\infty$. The inner-product of the space $L_2(\DD,\bR^t)$ is defined by
$$
\left<\ff,\bg\right>:=\sum_{j=1}^t\int_\DD f_j(\bx)g_j(\bx)d\bx, \ \ \mbox{for}\ \ \ff,\bg\in L_2(\DD,\bR^t).
$$
Below, we describe a continuous version of learning problem  \eqref{Basic-Min-Problem}. Given a function $\ff\in L_2(\DD,\bR^t)$, we wish to learn a deep neural network $\cN_n$ in the form of \eqref{TraditionalDNN} by solving $\{\bW_j^*,\bbb_j^*\}_{j=1}^n$ from the continuous minimization problem
\begin{equation}\label{Basic-Min-Problem-cont}
\min\left\{\|\ff(\cdot)-\cN_n(\{\bW_j,\bbb_j\}_{j=1}^n;\cdot)\|^2: \bW_j\in\bR^{m_j\times m_{j-1}}, \bbb_j\in\bR^{m_j}, j\in\bN_{n}\right\}.
\end{equation}
The function $\cN_n(\{\bW_j,\bbb_j\}_{j=1}^n;\cdot)$ is called a best approximation from the set $\Omega_n$ of deep neural networks having the form \eqref{DNN} to $\ff$. Note that the set $\Omega_n$ is a nonconvex closed subset of $L_2(\DD,\bR^t)$.

Learning a deep neural network from either discrete data or a continuous function boils down to finding the optimal weight matrices and bias vectors by solving minimization problem \eqref{Basic-Min-Problem} or \eqref{Basic-Min-Problem-cont}. 
The {\it single-grade} learning model learns the parameters of all layers together in one single grade. 
It is like to ask a college student to learn Linear Algebra without any backgrounds in High School Algebra or College Algebra, and learn all of these courses all together. Or image how difficult will be for a college freshman who has no College Algebra or Trigonometry knowledge to learn Calculus.
Computational challenges in solving these minimization problems  come  from their severe nonconvexity, which is the result of learning all parameters at once, besides their ill-posedness.  
Due to their severe nonconvexity, often their global minimizers cannot be found. Even using the mainstream method the stochastic gradient descent (Bottou 1998; Bottou and Bousquet, 2012) in solving minimization problems of this type, it is time consuming. Training a deep neural network in high dimensions requires even much more computing time.  

After a deep neural network is learned, if we realize that its accuracy is not satisfactory, then we need to train a new deep neural network with more layers. In this case, we have to start it over to train the new one in the single-grade learning model. This is not computationally efficient. One should have a model which allows updating the learned neural network to form a new one without training a complete new neural network. The multi-grade learning model to be proposed in the next section will address these issues.


\section{Multi-Grade Learning Models}
 
Motivated by the human education system, we introduce in this section multi-grade learning models to learn 
deep neural networks. 

We first consider learning a neural network that approximates $\ff\in L_2(\DD,\bR^t)$. Instead of learning it by the single-grade model \eqref{Basic-Min-Problem-cont}, we organize the ``curriculum'' in $l$ grades and allow the system to learn it in a multi-grade manner. Specifically, we choose $k_j\in \bN$, $j\in\bN_l$ so that $n-1=\sum_{j=1}^lk_j$. For each $k_j$, we choose a set of matrix widths $\{m_k: k=0,1,\dots,k_j\}$, which may be different for different $k_j$, and $m_{k_j}=t$. For simplicity, our notation does not indicate the dependence of the matrix widths $m_k$ on $k_j$. 

We first describe learning in grade 1. The goal of the first grade learning is to learn the neural network $\cN_{k_1}$ that takes the form of \eqref{DNN} with $n:=k_1$. To this end, we define the error function of grade 1 by
\begin{equation} \label{error1-G}
\be_1(\{\bW_j,\bbb_j\}_{j=1}^{k_1};\bx):=\ff(\bx)-\cN_{k_1}(\{\bW_j,\bbb_j\}_{j=1}^{k_1};\bx), \ \ \bx\in\bR^s,
\end{equation}
where $\{\bW_j,\bbb_j\}_{j=1}^{k_1}$ are parameters to be learned.
We find the parameters $\{\bW_{1,j}^*,\bbb_{1,j}^*\}_{j=1}^{k_1}$ from the optimization problem
\begin{equation}\label{min1-G}
\min\{\|\be_1(\{\bW_j,\bbb_j\}_{j=1}^{k_1};\cdot)\|^2: \bW_j\in\bR^{m_j\times m_{j-1}}, \bbb_j\in\bR^{m_j}, j\in\bN_{k_1}\},
\end{equation}
with $m_0:=s$. Once the optimal parameters $\{\bW_{1,j}^*,\bbb_{1,j}^*\}_{j=1}^{k_1}$ are learned, we let
$$
\ff_1(\bx)=\cN^*_{k_1}(\bx):=\cN_{k_1}(\{\bW_j^*,\bbb_j^*\}_{j=1}^{k_1};\bx),\ \ \bx\in\bR^s,
$$
which is the ``knowledge'' about $\ff$ that we have learned in grade 1 and we define the optimal error of grade 1 by setting
$$
\be^*_1(\bx):=\ff(\bx)-\ff_1(\bx),\ \ \mbox{for}\ \ \bx\in\bR^s.
$$
Usually, $\|\be^*_1\|$ is not small and thus the learning will continue. 

In grade 2, we will learn a shallow neural network on the top of $\cN^*_{k_1}$ from $\be^*_1$, which is the leftover from learning of grade 1. Specifically, we define the error function of grade 2 by
\begin{equation}\label{error2-G}
    \be_2(\{\bW_j, \bbb_j\}_{j=1}^{k_2};\bx):=\be_1^*(\bx)-(\cN_{k_2}(\{\bW_j,\bbb_j\}_{j=1}^{k_2};\cdot)\circ\cN^*_{k_1})(\bx),\ \ \bx\in\bR^s,
\end{equation}
where $\cN^*_{k_1}$ has been learned in grade 1 and its parameters will be fixed in learning of the future grades. The shallow neural network $\cN_{k_2}$, which has the form \eqref{DNN} with $n:=k_2$, will be learned in grade 2.
We find $\{\bW_{2,j}^*, \bbb_{2,j}^*\}_{j=1}^{k_1}$ from the optimization problem
\begin{equation}\label{min2-GG}
\min\{\|\be_2(\{\bW_{j}, \bbb_{j}\}_{j=1}^{k_2};\cdot)\|^2: \bW_j\in\bR^{m_j\times m_{j-1}}, \bbb_j\in\bR^{m_j},  j\in\bN_{k_2}\},
\end{equation}
with $m_0:=t$ and $m_{k_2}:=t$,  and we let
$$
\ff_2(\bx)=(\cN^*_{k_2}\circ\cN^*_{k_1})(\bx):=(\cN_{k_2}(\{\bW_{2,j}^*, \bbb_{2,j}^*\}_{j=1}^{k_1};\cdot)\circ\cN^*_{k_1})(\bx),\ \ 
\bx\in\bR^s.
$$
Note that $\ff_2$ is the newly learned neural network $\cN_{k_2}^*$ stacked on the top of the neural network $\cN_{k_1}^*$ learned in the previous grade. The optimal error of grade 2 is defined by
$$
\be^*_2(\bx):=\be^*_1(\bx)-\ff_2(\bx),\ \ \mbox{for}\ \ \bx\in\bR^s.
$$

Suppose that the neural networks $\cN^*_{k_i}$ of grades $i$, for $i\in\bN$,
have been learned and the optimal error of grade $i$ is given by $\be_{i}^*$.
We define the error function of grade $i+1$ by
$$
\be_{i+1}(\{\bW_j, \bbb_j\}_{j=1}^{k_{i+1}};\bx):=\be_{i}^*(\bx)-(\cN_{k_{i+1}}(\{\bW_j, \bbb_j\}_{j=1}^{k_{i+1}};\cdot)\circ\cN^*_{k_{i}}\circ\cdots\circ\cN^*_{k_1})(\bx),\ \ \bx\in\bR^s,
$$
where $\cN_{k_{i+1}}$ is a neural network to be learned in grade $i+1$ and it has the form \eqref{DNN} with $n:=k_{i+1}$.
We then find $\{\bW_{i+1,j}^*, \bbb_{i+1,j}^*\}_{j=1}^{k_{i+1}}$ from the optimization problem
\begin{equation}\label{minl-G}
\min\{\|\be_{i+1}(\{\bW_{j}, \bbb_{j}\}_{j=1}^{k_{i+1}};\cdot)\|^2: \bW_j\in\bR^{m_j\times m_{j-1}}, \bbb_j\in\bR^{m_j},  j\in\bN_{k_{i+1}}\},
\end{equation}
with $m_0:=t$ and $m_{k_{i+1}}:=t$, and we let
$$
\ff_{i+1}(\bx):=(\cN^*_{k_{i+1}}\circ\cN^*_{k_{i}}\circ\cdots\circ\cN^*_{k_1})(\bx), \ \ \bx\in\bR^s,
$$
where $\cN^*_{k_{i+1}}:=\cN_{k_{i+1}}(\{\bW_{i+1,j}^*, \bbb_{i+1,j}^*\}_{j=1}^{k_{i+1}};\cdot)$. Clearly, $\ff_{i+1}$ is a best approximation from the set 
\begin{equation}\label{Def-Omega(i+1)}
    \Omega_{i+1}:=\{\cN_{k_{i+1}}(\{\bW_j, \bbb_j\}_{j=1}^{k_{i+1}};\cdot)\circ\cN^*_{k_{i}}\circ\cdots\circ\cN^*_{k_1}: \bW_j\in\bR^{m_j\times m_{j-1}}, \bbb_j\in\bR^{m_j},  j\in\bN_{k_{i+1}}\}
\end{equation}
to $\be^*_i$.
Again, $\ff_{i+1}$ is the newly learned neural network $\cN_{k_{i+1}}^*$ stacked on the top of the neural network $\cN_{k_{i}}^*\circ\cdots\circ\cN_{k_1}^*$
learned in the previous grades. The optimal error of grade $i+1$ is defined by
$$
\be^*_{i+1}(\bx):=\be^*_{i}(\bx)-\ff_{i+1}(\bx), \ \ \mbox{for}\ \ \bx\in\bR^s.
$$
Finally, the $l$ grade learning model generates the neural network
\begin{equation}\label{Final-NN}
    \overline{\ff}_l:=\sum_{i=1}^l\ff_i.
\end{equation}
Unlike the neural network $\cN_n$ learned by \eqref{Basic-Min-Problem-cont}, the neural network $\overline{\ff}_l$ defined by \eqref{Final-NN} learned by the $l$-grade model is the accumulation of the  knowledge learned from all the grades. In each grade, the system learns based on the knowledge gained from learning of the previous grades. Mathematically, the neural network $\overline{\ff}_l$ is the superposition of all $l$ networks $\ff_i$, $i\in\bN_l$, learned in $l$ grades, and each $\ff_i$ is a shallow network learned in grade $i$ composed with the shallow networks learned from the previous grades. In general, the neural network $\overline{\ff}_l$ has a stairs-shape. We illustrate a three grade learning model in  Figure \ref{TwoGradeLearningModel}. Note that the neural network that results from a three grade learning model is the sum of three networks: the network (left) learned from grade 1 plus the network (center) learned from grade 2, which is stacked on the top of the network learned in grade 1, and plus the network (right) learned from grade 3, which is stacked on the top of the network learned in grade 2. The parts, bounded by dash lines, of the networks of grades 2 and 3 are the exact copy of the networks learned from the previous grades. They keep unchanged in training of a new grade.

\begin{figure}[H]
\centering
\includegraphics[width=0.8\textwidth, height=0.45\textwidth]{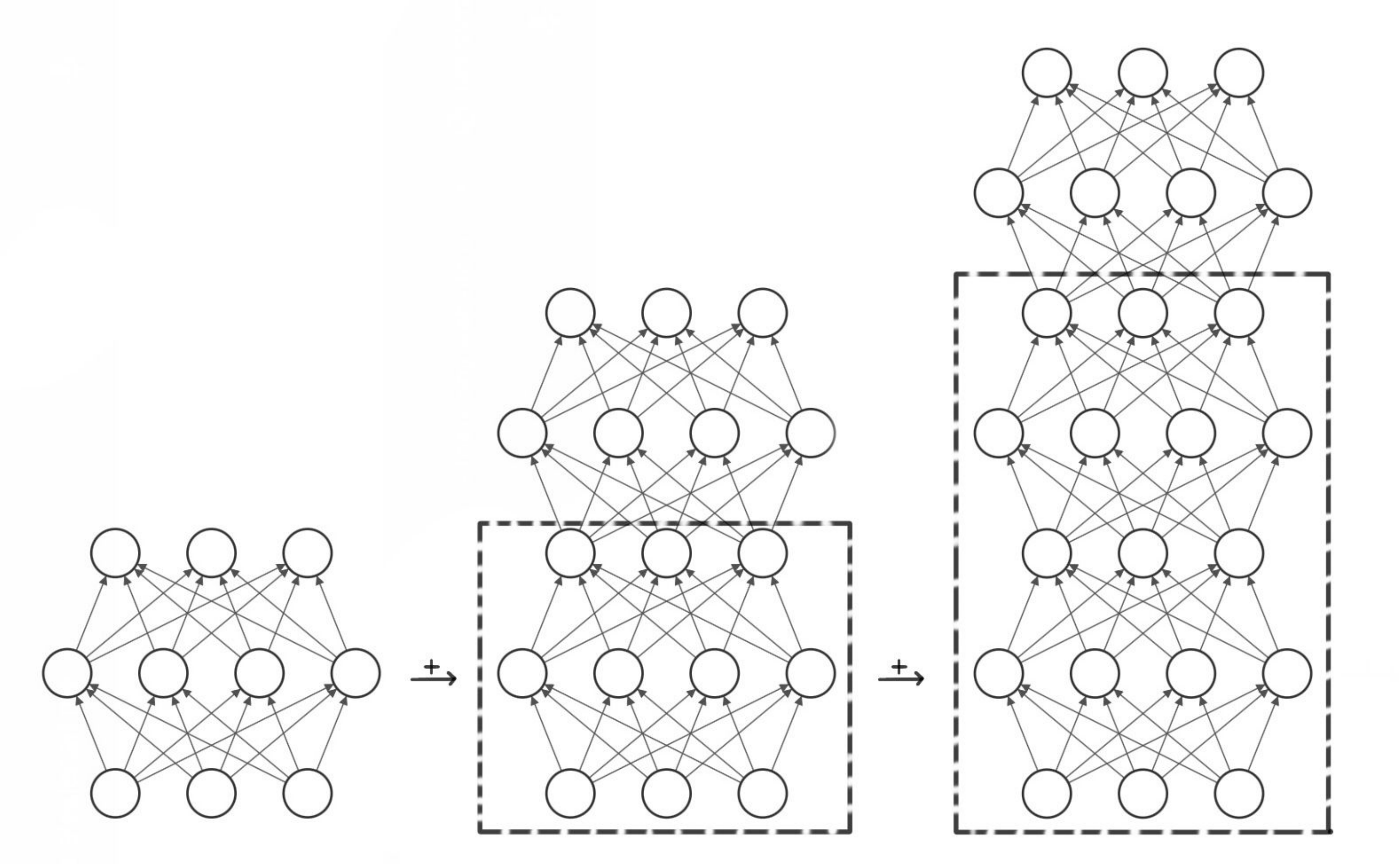}
\caption{Illustration of a three-grade learning model: The stairs-shape network - the superposition of three networks.}\label{TwoGradeLearningModel}
\end{figure}

From the structure of the neural network, we find it suitable for adaptive learning since when the norm of the optimal error of the current grade is not within a given tolerance we add a new grade without changing the neural networks learned in the previous grades.
In passing, we would like to point out that the neural network learned from the multi-grade learning model has the flavor of the adaptive mode decomposition originally introduced in (Huang et al., 1998), see also (Chen et al., 2006; Xu et al., 2006).

Learning a function from $m$ pairs of  discrete data points  $\{(\bx_i, \by_i)\}_{j=1}^m$ may be conducted in the same way, with the modification specified below. First of all,  the error functions are now defined by
\begin{equation*} \label{error1-GD}
\be_1(\{\bW_j, \bbb_j\}_{j=1}^{k_1};\bx_k):=\by_k-\cN_{k_1}  
(\{\bW_j,\bbb_j\}_{j=1}^{k_1};\bx_k), \ \ k\in \bN_m,
\end{equation*}
the optimal error of grade 1 by 
$$
\be^*_1(\bx_k):=\be_1(\{\bW_j^*, \bbb_j^*\}_{j=1}^{k_1};\bx_k), \ \ k\in\bN_m,
$$
and for $i=2,3,\dots,l$,
$$
\be_i(\{\bW_j, \bbb_j\}_{j=1}^{k_i};\bx_k):=\be_{i-1}^*(\bx_k)-(\cN_{k_{i}}
(\{\bW_j,\bbb_j\}_{j=1}^{k_1};\cdot)\circ\cN^*_{k_{i-1}}\circ\cdots\circ\cN^*_{k_1})(\bx_k),
\ \ k\in \bN_m,
$$
and the optimal error of grade $i$ by
$$
\be^*_i(\bx_k):=\be_i(\{\bW_j^*, \bbb_j^*\}_{j=1}^{k_i};\bx_k), \ \ k\in\bN_m.
$$
The $L_2$-norm square in the minimization problems \eqref{min1-G}, \eqref{min2-GG} and \eqref{minl-G} is replaced by the discrete forms 
\begin{equation}\label{Error-Discrete1}
    \|\be_1(\{\bW_j, \bbb_j\}_{j=1}^{k_1};\cdot)\|_m^2:=\sum_{k=1}^m\|\by_k-\cN_{k_1}
(\{\bW_j,\bbb_j\}_{j=1}^{k_1};\bx_k)\|_{\ell_2}^2
\end{equation}
and for $i=2,3,\dots,l$,
\begin{equation}\label{Error-Discretei}
\|\be_i(\{\bW_j, \bbb_j\}_{j=1}^{k_i};\cdot)\|_m^2:=\sum_{k=1}^m\|\be_{i-1}^*(\bx_k)-(\cN_{k_{i}}
(\{\bW_j,\bbb_j\}_{j=1}^{k_1};\cdot)\circ\cN^*_{k_{i-1}}\circ\cdots\circ\cN^*_{k_1})(\bx_k)\|_{\ell_2}^2.
\end{equation}
Everything else remains the same as the function approximation case.

Learning a solution of an operator equation, such as initial/boundary value problems of linear/nonlinear differential equations, integral equations and functional equations, can be carried out in the same way. Note that learning a solution of nonlinear differential equations by the single-grade learning model was recently considered by  (Raissi 2018; Xu and Zeng, 2023).
Let $\cG$ denote a nonlinear operator and we consider the operator equation
\begin{equation}\label{Operator-Equ}
    \cG(\ff)=0,
\end{equation}
where $\ff:\bR^s\to\bR^t$ is a solution to be learned. 
For an initial value problem or a boundary value problem of a differential equation, the operator $\cG$ includes the initial value conditions or/and the boundary value conditions.

We will learn a solution $\ff$ of equation \eqref{Operator-Equ} by the $l$-grade learning model. We define the error function of grade 1 by
$$
\be_i(\{\bW_j,\bbb_j\}_{j=1}^{k_1};\cdot):=\cG(\cN_{k_1}(\{\bW_j,\bbb_j\}_{j=1}^{k_1};\cdot)).
$$
We then find $\{\bW_{1,j}^*,\bbb_{1,j}^*\}_{j=1}^{k_1}$ by solving 
\begin{equation}\label{min1-op-G}
   \min
   \left\{\sum_{k=1}^m\|\be_i(\{\bW_j,\bbb_j\}_{j=1}^{k_1};\bx_k))\|_{\ell_2}^2: \bW_j\in\bR^{m_j\times m_{j-1}}, \bbb_j\in\bR^{m_j}, j\in\bN_{k_1}\right\},
\end{equation}
with $m_0:=s$ and $m_{k_1}:=t$, and obtain the knowledge of grade 1 for the solution of equation \eqref{Operator-Equ}, which has the form 
\begin{equation}\label{ff1-G}
    \ff_1:=\cN_{k_1}^*=\cN_{k_1}(\{\bW_{1,j}^*,\bbb_{1,j}^*\}_{j=1}^{k_1};\cdot).
\end{equation}
For $i=2,3,\dots,l$, we successively define the error function of grade $i$ by
$$
\be_i(\{\bW_j,\bbb_j\}_{j=1}^{k_i};\bx):=\cG\left(\sum_{\mu=1}^{i-1}\ff_\mu(\bx)+(\cN_{k_i}(\{\bW_j,\bbb_j\}_{j=1}^{k_i};\cdot)\circ\ff_{i-1}\circ\cdots\circ\ff_1)(\bx)\right)
$$
and find $\{\bW_{i,j}^*,\bbb_{i,j}^*\}_{j=1}^{k_i}$ by 
solving the minimization problems
\begin{equation}\label{min-i-op-G}
  \min
   \left\{\sum_{k=1}^m\left\|\be_i(\{\bW_j,\bbb_j\}_{j=1}^{k_i};\bx_k)\right\|_{\ell_2}^2: 
   \bW_j\in\bR^{m_j\times m_{j-1}}, \bbb_j\in\bR^{m_j}, j\in\bN_{k_i}\right\},
\end{equation}
with $m_0:=t$ and $m_{k_i}:=t$.
We let 
\begin{equation}\label{ffi-G}
    \ff_i:=\cN_{k_i}(\{\bW_{i,j}^*,\bbb_{i,j}^*\}_{j=1}^{k_i};\cdot)\circ\ff_{i-1}\circ\cdots\circ\ff_1, \ \ i=2,3,\dots l.
\end{equation}
We define the optimal error of grade $i$ by
$
\be_i^*:=\be_i(\{\bW_j^*,\bbb_j^*\}_{j=1}^{k_i};\bx).
$
The learning stops if $\|\be^*_i\|_{\ell_2}$ is within a given tolerant error bound or we reach the maximal grade number $l$.
When the $l$-grade learning process is completed, we obtain the neural network
\begin{equation}\label{FinalNetwork}
    \overline{\ff}_l:=\sum_{i=1}^l\ff_i,
\end{equation}
which serves as an approximate solution of the operator equation \eqref{Operator-Equ}.

The proposed multi-grade learning is particularly suitable for adaptive solutions of operator equations. In this case, we have a nature posterior error  $\|\be^*_i\|_{\ell_2}$ to control the stop of learning. Moreover, the multi-grade learning model allows us to add a new term $\ff_{i+1}$ to $\overline{\ff}_i$ if  $\|\be^*_i\|_{\ell_2}$ is not within the given tolerance, without starting over to learn the entire neural network.

The proposed multi-grade model works for other norms. As an example, we consider classification using the categorical cross entropy loss function (Brownlee 2019). Given data $\{(\bx_k,y_k)\}_{k=1}^m\subset \bR^s\times \bN_t$, where $\bx_k$ are observations and $y_k$ are labels, we wish to find a classifier $\cC: \bR^s\to\bN_t$ such that it correctly classifies the training data pairs and as well as new data points not in the given data set.
For a given true label $\by:=[y_1,\dots,y_t]^\top\in\bR^r$ and its prediction $\hat\by:=[\hat y_1,\dots, \hat y_t]^\top$, we define the individual categorical cross entropy by
\begin{equation}\label{Entropy}
    L(\by,\hat\by):=\sum_{k=1}^ty_k\ln (\hat y_k).
\end{equation}
The standard deep learning model for classification is to find $\bW_j^*\in\bR^{m_j\times m_{j-1}}$, $\bbb_j^*\in\bR^{m_j}$, $j\in\bN_{n}$ with $m_0:=s$ and $m_n:=t$ such that
\begin{equation}\label{classification}
\min
   \left\{-\frac{1}{m}\sum_{k=1}^mL(\by_k, \cN_n(\{\bW_j,\bbb_j\}_{j=1}^n; \bx_k)): 
   \bW_j\in\bR^{m_j\times m_{j-1}}, \bbb_j\in\bR^{m_j}, j\in\bN_{n}\right\},
\end{equation}
where $L$ is defined by \eqref{Entropy}.  When the dimension $t$ of the vector-valued prediction function is large, solving the nonconvex optimization problem  \eqref{classification} with a large number $n$ of layers is a challenging task if it is not impossible. The proposed multi-grade model can overcome this computational challenge since instead of solving one nonconvex optimization problem of a large size for one deep neural network, we solve nonconvex optimization problems of small sizes for several shallow neural networks.

We now describe the $l$-grade learning model for the classification problem. We first learn $\bW_{1,j}^*\in\bR^{m_j\times m_{j-1}}$, $\bbb_{1,j}^*\in\bR^{m_j}$, $j\in\bN_{k_1}$ with $m_0:=s$ and $m_{k_1}:=t$ such that
\begin{equation}\label{classification1}
   \min\left\{-\frac{1}{m}\sum_{k=1}^mL(\by_k, \cN_{k_1}(\{\bW_j,\bbb_j\}_{j=1}^{k_1}; \bx_k)): 
   \bW_j\in\bR^{m_j\times m_{j-1}}, \bbb_j\in\bR^{m_j}, j\in\bN_{k_1}\right\},
\end{equation}
and define prediction $\ff_1$ of grade 1 in the same way as in \eqref{ff1-G}.
For $i=2,3,\dots, l$, we define possible predictions of grade $i$ by
$$
\tilde \by_{i}(\{\bW_j,\bbb_j\}_{j=1}^{k_i};\bx_k) :=\sum_{\mu=1}^{i-1}\ff_\mu(\bx_k)+ (\cN_{k_i}(\{\bW_j,\bbb_j\}_{j=1}^{k_i};\cdot)\circ\ff_{i-1}\circ\cdots\circ\ff_1)(\bx_k).
$$
We then successively learn $\bW_{i,j}^*\in\bR^{m_j\times m_{j-1}}$, $\bbb_{i,j}^*\in\bR^{m_j}$, $j\in\bN_{k_i}$ with $m_0=m_{k_i}:=t$ such that
\begin{equation}\label{classification2}
   \min\left\{-\frac{1}{m}\sum_{k=1}^mL(\by_k,  \tilde \by_{i}(\{\bW_j,\bbb_j\}_{j=1}^{k_i};\bx_k)): 
   \bW_j\in\bR^{m_j\times m_{j-1}}, \bbb_j\in\bR^{m_j}, j\in\bN_{k_1}\right\},
\end{equation}
and define prediction $\ff_i$ of grade $i$ by setting 
$$
\ff_i(\bx):= 
(\cN_{k_i}(\{\bW_{i,j}^*,\bbb_{i,j}^*\}_{j=1}^{k_i};\cdot)\circ\ff_{i-1}\circ\cdots\circ\ff_1)(\bx), \ \ \bx\in\bR^s.
$$
When all $l$ grades of learning are completed, we obtain our prediction $\overline{\ff}_l$ which has the same form as in \eqref{FinalNetwork}.

\section{Analysis of the Multi-Grade Learning Model}

We analyze in this section the proposed multi-grade learning model. The main result is the optimal error of a multi-grade learning model decreases as the number of grades increases. To this end, we first present several useful observations regarding a best approximation from a nonconvex set, which are interesting in their own right. 

Let $\cH$ denote a Hilbert space and $\Omega$ a  closed set (not necessarily convex) in $\cH$. We wish to understand a best approximation from $\Omega$ to an $\ff\in\cH$.
We first present a necessary condition and a sufficient condition for an element to be a best approximation. To this end, we define a star-shape set centered at a given $\bg_0\in\Omega$ by
\begin{equation}\label{Set-Omega}
    \Omega(\bg_0):=\{\bg\in\Omega: \lambda \bg+(1-\lambda)\bg_0\in\Omega \ \mbox{for all}\ \lambda\in(0,1)\}.
\end{equation}
Clearly, a point $\bg$ of $\Omega$ is in $\Omega(\bg_0)$ if and only if the line segment connecting the two points $\bg_0$ and $\bg$ is completely located in $\Omega$.

\begin{theorem}\label{nonconvex-aproximation}
Suppose that $\cH$ is a Hilbert space and $\Omega$ is a  closed set in $\cH$. Let $\ff\in\cH\setminus \Omega$.

(i) If $\bg_0\in\Omega$ is a best approximation from $\Omega$ to $\ff$, then for all $\bg\in\Omega(\bg_0)$, 
\begin{equation}\label{Angle-Condition4.1}
    \left<\ff-\bg_0, \bg-\bg_0\right>\leq 0.
\end{equation}

(ii) If $\bg_0\in\Omega$ and \eqref{Angle-Condition4.1} holds for all $\bg\in\Omega$,  then $\bg_0$ is a best approximation from $\Omega$ to $\ff$.
\end{theorem}
\begin{proof}
(i) Assume that \eqref{Angle-Condition4.1} does not hold for all $\bg\in \Omega(\bg_0)$. Then, there exists some $\bg\in\Omega(\bg_0)$ such that 
$$
\left<\ff-\bg_0, \bg-\bg_0\right>>0.
$$
Since $\bg\in\Omega(\bg_0)$, by the definition of $\Omega(\bg_0)$, we observe that 
$$
\bg_\lambda:=\lambda \bg+(1-\lambda)\bg_0\in \Omega, \ \ \mbox{for all}\ \ \lambda\in (0,1).
$$
Hence, we have that
$$
\ff-\bg_\lambda=\ff-\bg_0-\lambda(\bg-\bg_0).
$$
This, with a direct computation, leads to 
$$
\|\ff-\bg_\lambda\|^2=
\|\ff-\bg_0\|^2-\lambda(2\left<\ff-\bg_0, \bg-\bg_0\right>-\lambda\|\bg-\bg_0\|^2).
$$
Upon choosing $\lambda$ to satisfy 
$$
0<\lambda<\min\left\{1, \frac{2\left<\ff-\bg_0, \bg-\bg_0\right>}{\|\bg-\bg_0\|^2} \right\},
$$
we obtain that $\|\ff-\bg_\lambda\|<\|\ff-\bg_0\|$. That is, $\bg_0$ is not a best approximation from $\Omega$ to $\ff$, a contradiction. The contradiction implies that  \eqref{Angle-Condition4.1} must hold for all $\bg\in \Omega(\bg_0)$.

(ii) By hypothesis, for all $\bg\in\Omega$, we observe that
$$
\begin{aligned}
     \|\ff-\bg_0\|^2&=\left<\ff-\bg_0,\ff-\bg\right>+\left<\ff-\bg_0,\bg-\bg_0\right>\\
    &\leq \left<\ff-\bg_0,\ff-\bg\right>\\
    &\leq  \|\ff-\bg_0\|\|\ff-\bg\|.
\end{aligned}
$$
This implies that 
$$
\|\ff-\bg_0\| \leq \|\ff-\bg\|,\ \  \mbox{for all}\ \ \bg\in\Omega.
$$
That is,  $\bg_0$ is a best approximation from $\Omega$ to $\ff$.
\end{proof}

Note that in general,  $\Omega(\bg_0)\neq \Omega$, and thus, the necessary condition stated in Theorem \ref{nonconvex-aproximation} is not the same as the sufficient condition. However, under some condition, they can be the same.
The following special result is a corollary of Theorem \ref{nonconvex-aproximation}.

\begin{corollary}\label{SpecialResult}
    Suppose that $\cH$ is a Hilbert space and $\Omega$ is a  closed set in $\cH$. Let $\ff\in\cH\setminus \Omega$. If  $\bg_0\in\Omega$ such that  $\Omega(\bg_0)= \Omega$,
then $\bg_0\in\Omega$ is a best approximation from $\Omega$ to $\ff$ if and only if 
\begin{equation}\label{Angle-Condition}
    \left<\ff-\bg_0, \bg-\bg_0\right>\leq 0, \ \ \mbox{for all} \ \ \bg\in\Omega(\bg_0).
\end{equation}
\end{corollary}
\begin{proof}
When  $\Omega(\bg_0)= \Omega$, the two conditions in Items (i) and (ii) of Theorem \ref{nonconvex-aproximation} are identical, which gives a characterization of a best approximation in this special case.
\end{proof} 

In particular, when $\Omega$ is a convex set, for any $\bg_0\in\Omega$, we have that $\Omega=\Omega(\bg_0)$. Thus, Corollary \ref{SpecialResult} reduces to the well-known  characterization of a best approximation from a convex set, see, for example, (Deutsch 2001). For this reason, Theorem \ref{nonconvex-aproximation} is a natural extension of the characterization of a best approximation from a convex set to that from a nonconvex set. In fact,  Theorem \ref{nonconvex-aproximation} leads to a characterization of a best approximation from a nonconvex set, which we present next.
For a given $\ff\in\cH$ and a fixed $\bg_0\in\Omega$, we let $r_0:=\|\ff-\bg_0\|$ and define the open ball
$$
B(\ff,r_0):=\{\bg\in \cH: \|\ff-\bg\|<r_0\}.
$$
Moreover, for a given $\bg_0\in\Omega$, we let 
$$
\tilde\Omega(\bg_0):=\Omega\setminus \Omega(\bg_0)
$$
and 
$$
d(\ff, \tilde\Omega(\bg_0)):=\inf\{\|\ff-\bg\|:
    \bg\in\tilde\Omega(\bg_0)\}.
$$

\begin{theorem}\label{CharacterizationofBestApprox}
Suppose that $\cH$ is a Hilbert space and $\Omega$ is a  closed set in $\cH$. Let $\ff\in\cH\setminus \Omega$ and $\bg_0\in\Omega$.
The following statements are equivalent:

(i) The element $\bg_0$ is a best approximation from $\Omega$ to $\ff$.

(ii) There holds the equation
\begin{equation}\label{Characterization}
    B(\ff,r_0)\cap \Omega = \emptyset.
\end{equation}

(iii) The element $\bg_0$ satisfies the conditions
\begin{equation}\label{local-Angle-Condition}
    \left<\ff-\bg_0, \bg-\bg_0\right>\leq 0,\ \ \mbox{for all}\ \ 
    \bg\in\Omega(\bg_0)
\end{equation}
and
\begin{equation}\label{local-Angle-Condition2}
    \|\ff-\bg_0\|\leq d(\ff, \tilde\Omega(\bg_0)).
\end{equation}
\end{theorem}
\begin{proof}
We first show that Items (i) and (ii) are equivalent. Suppose that equation \eqref{Characterization} holds. Let $\bg\in\Omega$ be arbitrary. Because of \eqref{Characterization}, 
we find that $\bg\notin B(\ff, r_0)$. Hence, we have that $\|\ff-\bg\|\geq r_0$. Equivalently, we find that 
\begin{equation}\label{BestA}
    \|\ff-\bg_0\|\leq \|\ff-\bg\|,\ \ \mbox{for all}\ \  \bg\in\Omega.
\end{equation}
That is, $\bg_0$ is a best approximation from $\Omega$ to $\ff$.
Conversely, suppose that $\bg_0$ is a best approximation from $\Omega$ to $\ff$. Then, inequality \eqref{BestA} holds. This implies that for all $\bg\in\Omega$, $\bg\notin B(\ff, r_0)$. That is, equation \eqref{Characterization} holds.

We next show that Items (i) and (iii) are equivalent.
Suppose that $\bg_0\in\Omega$ is a best approximation from $\Omega$ to $\ff$. By Item (i) of Theorem \ref{nonconvex-aproximation}, we have that \eqref{local-Angle-Condition}. Condition \eqref{local-Angle-Condition2} follows from the definition of $\bg_0$ being a best approximation from $\Omega$ to $\ff$ and the fact that $\tilde\Omega(\bg_0)\subseteq \Omega$. 
Conversely, suppose that  $\bg_0\in\Omega$ satisfies both \eqref{local-Angle-Condition} and \eqref{local-Angle-Condition2}. By Item (ii) of Theorem  \ref{nonconvex-aproximation}, $\bg_0$ is a best approximation from $\Omega(\bg_0)$ to $\ff$. That is, 
we have that 
$$
\|\ff-\bg_0\|\leq \|\ff-\bg\|,\ \ \mbox{for all}\ \ 
    \bg\in\Omega(\bg_0).
$$
This together with condition \eqref{local-Angle-Condition2} implies that 
$$
\|\ff-\bg_0\|\leq \|\ff-\bg\|,\ \ \mbox{for all}\ \ 
    \bg\in\Omega:=\Omega(\bg_0)\cup\tilde\Omega(\bg_0).
$$
In other words, $\bg_0$ is a best approximation from $\Omega$ to $\ff$.
\end{proof}

Figure \ref{Projection-to-NonconvexSet} illustrates the characterization of a best approximation from a nonconvex set to an element in a Hilbert space. As shown in the figure, $\Omega$ is a nonconvex set of $\cH$, $\ff$ is an element in $\cH$, not in $\Omega$, and $\bg_0\in \Omega$ is a best approximation from $\Omega$ to $\ff$, the shaded domain is the set $\Omega(\bg_0)$, and $\tilde\Omega(\bg_0)$ is the rest in $\Omega$ not in  $\Omega(\bg_0)$. Intuitively, the set $\Omega(\bg_0)$ consists of all points in $\Omega$ which one can ``see'' from the point $\bg_0$. Clearly, $\bg_0$ is a best approximation from $\Omega(\bg_0)$ to $\ff$ and the distance of $\ff$ to $\tilde\Omega(\bg_0)$ is larger than $r_0:=\|\ff-\bg_0\|$. Therefore, $\bg_0$ is a best approximation from $\Omega$ to $\ff$. 

\begin{figure}[H]
\centering
\includegraphics[width=0.5\textwidth, height=0.42\textwidth]{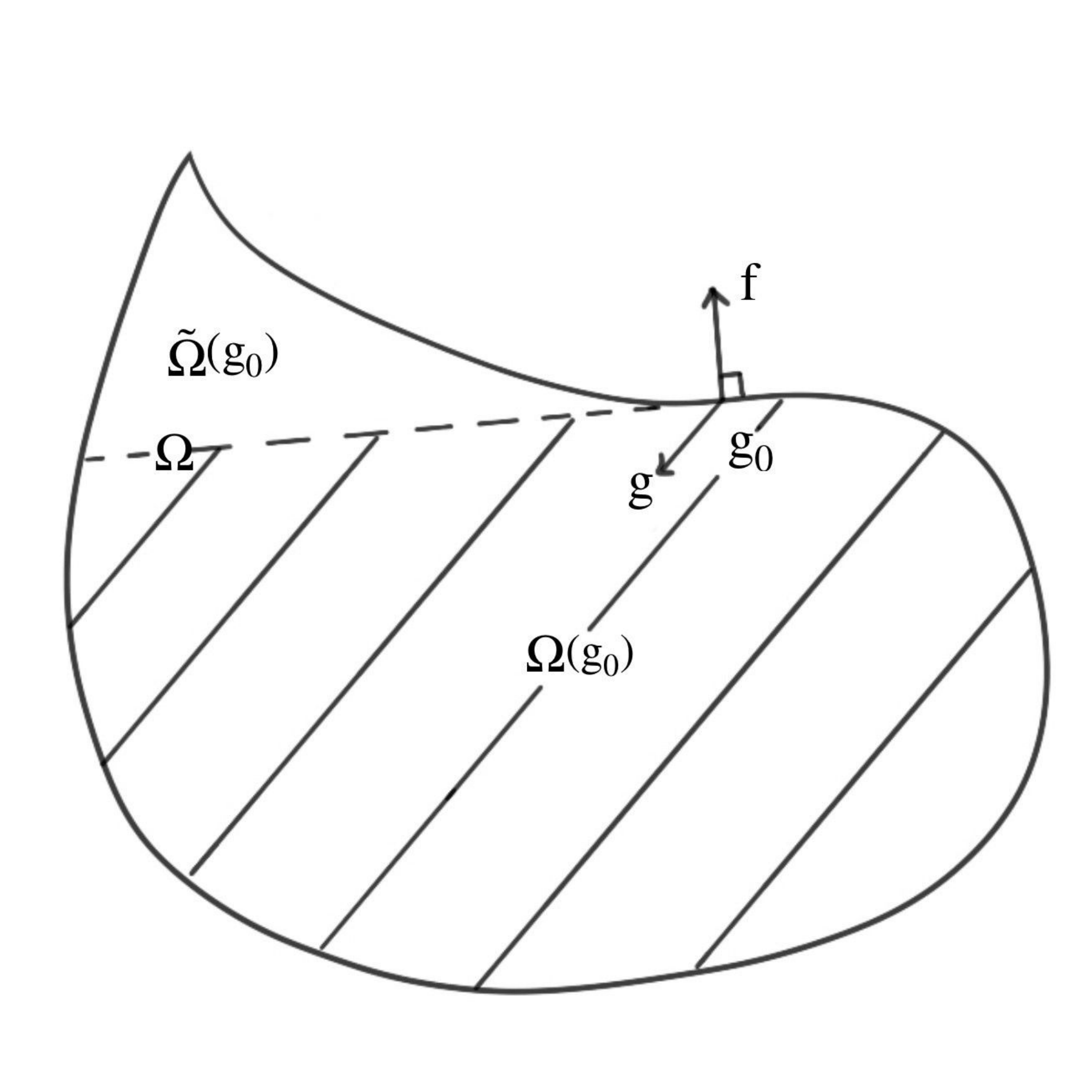}
\caption{Projection to a Nonconvex Set}\label{Projection-to-NonconvexSet}
\end{figure}

We now return to the multi-grade learning model.
In the next theorem, we represent the original function $\ff$ in terms of the ``knowledge'' (neural networks) learned in all grades and show that the norms of the optimal errors are  nonincreasing.

\begin{theorem}\label{TheoremforGG}
Let $\ff\in L_2(\DD, \bR^t)$. The following statements hold:

(i) For all $l\in\bN$,
\begin{equation}\label{ff-expression}
    \ff(\bx)=\sum_{i=1}^l(\cN^*_{k_{i}}\circ\cN^*_{k_{i-1}}\circ\cdots\circ\cN^*_{k_1})(\bx)+\be^*_l(\bx), \ \ \bx\in\bR^s.
\end{equation}

(ii) For all $i\in\bN$,
\begin{equation}\label{Pythagorean-Ext}
    \|\be_i^*\|^2\geq \|\ff_{i+1}\|^2+\|\be^*_{i+1}\|^2.
\end{equation}

(iii) For all $i\in\bN$,
\begin{equation}\label{Nonincreasingness}
    \|\be^*_{i+1}\|\leq \|\be^*_i\|,
\end{equation}
and the sequence $\|\be_i^*\|$, $i\in\bN$, has a nonnegative limit.

(iv) For each $i\in\bN$, either $\ff_{i+1}={\bf 0}$ or 
\begin{equation}\label{SDecreasingness}
    \|\be^*_{i+1}\|< \|\be^*_i\|.
\end{equation}
\end{theorem}
\begin{proof}
(i) By induction on $l$, we can express the original function $\ff$ as
\begin{equation}\label{ff-expression-0}
    \ff=\sum_{i=1}^l\ff_i+\be^*_l.
\end{equation}
By the $l$-grade learning model, we find for each $i\in\bN_l$ that
$$
\ff_i(\bx):=(\cN^*_{k_{i}}\circ\cN^*_{k_{i-1}}\circ\cdots\circ\cN^*_{k_1})(\bx), \ \ \bx\in\bR^s.
$$
Upon substituting the equation above into the right-hand-side of equation \eqref{ff-expression-0}, we obtain formula \eqref{ff-expression}.

(ii) Note that for all $i\in\bN$, $\be^*_i=\ff_{i+1}+\be^*_{i+1}$.  Hence, we have that
$$
\|\be^*_i\|^2=\left<\ff_{i+1}+\be^*_{i+1}, \ff_{i+1}+\be^*_{i+1}\right>.
$$
Expanding the inner-product on the right-hand-side of the equation above yields that
\begin{equation}
 \label{Identityy}
\|\be^*_i\|^2=\|\ff_{i+1}\|^2+\|\be^*_{i+1}\|^2+2\left<\ff_{i+1}, \be^*_{i+1}\right>.
\end{equation}
We next apply Theorem \ref{nonconvex-aproximation} with $\cH:=L_2(\DD,\bR^t)$,  $\ff:=\be_i^*$, $\bg_0:=\ff_{i+1}$ and $\bg:={\bf 0}$. According to the definition \eqref{Def-Omega(i+1)} of $\Omega_{i+1}$, we see that 
$$
{\bf 0}\in\Omega_{i+1}\ \ \mbox{and}\ \ \lambda {\bf 0}+(1-\lambda)\ff_{i+1}\in\Omega_{i+1}, \ \ \mbox{for all}\ \ \lambda\in (0,1).
$$
Hence, by the definition \eqref{Set-Omega} of set $\Omega_{i+1}(\ff_{i+1})$ with $\Omega:=\Omega_{i+1}$ and $\bg_0:=\ff_{i+1}$, we conclude that ${\bf 0}\in \Omega_{i+1}(\ff_{i+1})$. By Item (i) of Theorem \ref{nonconvex-aproximation}, we observe that 
$$
\left<\ff_{i+1}, \be^*_{i+1}\right>=\left<\ff_{i+1}, \be^*_{i}-\ff_{i+1}\right>\geq 0.
$$
This together with \eqref{Identityy} establishes Item (ii).

(iii) Inequality \eqref{Nonincreasingness} follows directly from Item (ii).
The existence of the limit of the sequence $\|\be_i^*\|$, $i\in\bN$, is a direct consequence of its nonincreasingness  \eqref{Nonincreasingness} and nonnegativity.

(iv) If $\ff_{i+1}\neq {\bf 0}$, then $\|\ff_{i+1}\|>0$.  By inequality \eqref{Pythagorean-Ext}, we obtain \eqref{SDecreasingness}.
\end{proof}

We observe from \eqref{ff-expression} that the network $\ff_l$ learned from the $l$-grade model has much redundancy: $\cN_{k_i}^*$ appears in the network $l-i+1$ times, for $i\in\bN_l$. The redundancy increases significantly the expressiveness of the neural network.

For multi-grade learning of discrete data, we have results similar to those of Theorem 
\ref{TheoremforGG}.

\begin{theorem}\label{TheoremforGG-discrete}
 (i)  There holds for  $l\in\bN$ that
\begin{equation}\label{ff-expression-m}
    \by_k=\sum_{i=1}^l(\cN^*_{k_{i}}\circ\cN^*_{k_{i-1}}\circ\cdots\circ\cN^*_{k_1})(\bx_k)+\be^*_l(\bx_k), \ \ k\in\bN_m.
\end{equation}

(ii) For all $i\in\bN$,
\begin{equation}\label{Pythagorean-Ext-m}
    \|\be_i^*\|_m^2\geq \|\ff_{i+1}\|_m^2+\|\be^*_{i+1}\|_m^2.
\end{equation}

(iii) For all $i\in\bN$,
\begin{equation}\label{Nonincreasingness-m}
    \|\be^*_{i+1}\|_m\leq \|\be^*_i\|_m,
\end{equation}
and the sequence $\|\be_i^*\|_m$, $i\in\bN$, has a nonnegative limit.

(iv) For each $i\in\bN$, either $\ff_{i+1}={\bf 0}$ or 
\begin{equation}\label{SDecreasingness-m}
    \|\be^*_{i+1}\|_m< \|\be^*_i\|_m.
\end{equation}
\end{theorem}

Item (i) of Theorem \ref{TheoremforGG-discrete} indicates that for an $l$-grade learning model, the neural network $\sum_{i=1}^l\cN^*_{k_{i}}\circ\cN^*_{k_{i-1}}\circ\cdots\circ\cN^*_{k_1}$ ``interpolates'' the data points $(\bx_k,\by_k)$, for $k\in\bN_m$, up to the error $\be^*_l(\bx_k)$.

Theorems \ref{TheoremforGG} and \ref{TheoremforGG-discrete}
ensure that every grade of a multi-grade model reduces the global approximation error if the network learned from the grade is nontrivial.

\section{Implementation}

We discuss in this section several crucial issues in implementing the multi-grade learning model proposed in the previous sections. These issues include regularization to
overcome over-fitting, removal of output layers for each grade,  design of the grades, and the posterior error for adaptive computation.

Solving the optimization problems involved in the multi-grade learning models described earlier may suffer from over-fitting as solving the one involved in the single-grade learning model does (Rice et al., 2020). When over-fitting occurs, regularization may be necessary. One may use the $\ell_2$ regularization.
We prefer employing the $\ell_1$ regularization, because it can alleviate over-fitting and at the same time  promote the sparsity (Liu et al., 2022). Specifically, we add the vector $\ell_1$-norm of the weight matrix to the minimization problems for our multi-grade learning models. We first recall the definition of the vector $\ell_1$-norm of a matrix. Given a matrix $\bA:=[a_{ij}]\in \bR^{t\times s}$, its vector $\ell_1$-norm is defined by
$$
\|\bA\|_1:=\sum_{i=1}^t\sum_{j=1}^s|a_{ij}|,
$$
the sum of the absolute values of all entries of the matrix.
Note that the vector $\ell_1$ norm, different from the standard $\ell_1$ matrix norm, see (Horn and Johnson, 2012), is the $\ell_1$-norm of a matrix when it is treated as a vector in $\bR^{ts}$.

We next take the minimization problem \eqref{min-i-op-G} as an example to illustrate how  a regularization term is added.
Instead of solving minimization problem  \eqref{min-i-op-G}, we find $\{\bW_{i,j}^*, \bbb_{i,j}^*\}_{j=1}^{k_i}$ by solving the following regularization problem
\begin{equation}\label{Regularization-Problem}
    \min\left\{\sum_{k=1}^m\|\be_i(\{\bW_{j}, \bbb_{j}\}_{j=1}^{k_i};\bx_k)\|_{\ell_2}^2 
+\sum_{j=1}^{k_i}\lambda_j\|\bW_j\|_1: \bW_j\in\bR^{m_j\times m_{j-1}}, \bbb_j\in\bR^{m_j},  j\in\bN_{k_i}\right\},
\end{equation}
where $\lambda_j>0$ are the regularization parameters.
A similar regularization approach was used in (Xu and Zeng, 2022) in training deep neural networks. Parameter choice strategies for the $\ell_1$ regularization proposed in (Liu et al., 2022) is suitable for the case when the fidelity term is convex. One needs to extend the strategies to treat the regularization problem \eqref{Regularization-Problem} when the fidelity term is nonconvex.

In the multi-grade learning models described in section 3, each grade has an output layer for the purpose of defining the error function for the grade. We recommend that after the shallow neural network of the grade is learned, its output layer be removed before entering the next grade of learning. For example, when defining the error function $\be_2(\{\bW_j, \bbb_j\}_{j=1}^{k_2};\bx)$ of grade 2, instead of using \eqref{error2-G}, we employ the modified error function
\begin{equation*}\label{error2-G-6}
    \tilde\be_2(\{\bW_j, \bbb_j\}_{j=1}^{k_2};\bx):=\be_1^*(\bx)-(\cN_{k_2}(\{\bW_j,\bbb_j\}_{j=1}^{k_2};\cdot)\circ\tilde\cN^*_{k_1})(\bx),\ \ \bx\in\bR^s,
\end{equation*}
where $\tilde\cN^*_{k_1}$ is the shallow neural network of grade 1 after removing the output layer and the column size of $\cN_{k_2}(\{\bW_j,\bbb_j\}_{j=1}^{k_2};\cdot)$ is changed accordingly. We make the same modification for the error functions of grades $i>2$.
In this way,  the learning will be more efficient.

Proper defining the ``grades'' 
is vital to the success of a multi-grade learning model.  It depends on specific applications. In general, we should take the following factors into consideration: In each of the grades, we prefer to training a shallow neural network because it is more efficient to learn a shallow neural network than a deep one. However, the more grades we have, the more overhead we pay. The overhead includes output layers of the grades. Hence, we need to balance the depth of the shallow neural networks and the overhead.

In the first a few grades, it is not necessary to learn networks of these grades with high accuracy since the leftover will be learned in future grades. If one tries to learn them with high accuracy in the first a few grades, one may end up with spending too much times on these grades.

One of the advantages of the multi-grade learning is that it is suitable for adaptive computation. The multi-grade learning model naturally allows us to continuously add new grades without starting over which a single-grade learning model would do. This requires the availability of a posterior error. For most of learning problems, the error function for each grade can serve as an innate posterior error. 

Finally, we comment that one can use more than one activation functions to build a neural network. Numerical examples presented in section 6 show the advantages of using different activation functions in one neural network. Different activation functions serve different purposes. Using different activation functions will not change the theoretical results of Theorems
\ref{TheoremforGG} and \ref{TheoremforGG-discrete}.

\section{Numerical Examples}

We present in this section three {\it proof-of-concept} numerical examples to demonstrate the robustness of the proposed multi-grade learning model in comparison to the classical single-grade learning model. We will learn three functions, which are either oscillatory or singular. For each of the functions, we consider both noise free and noisy cases.

All the experiments reported in this section are performed with Python on an Intel Core i7 CPU with 1.80GHz, 16 Gb RAM and 12 Gb NVIDIA GeForce MX150 GPU.

Below, we describe our experiment data for a given function $f$.

\noindent\textbf{Training data}: $\{(x_n, y_n)\}_{n = 1}^N \subset [a, b] \times \mathbb{R}$, where  $N := 5,000$, $x_n$'s are equally spaced on $[a, b]$, and given $x_n$, the corresponding $y_n$ is computed by
$y_n = f(x_n)+e_n$. For the noise free cases, $e_n=0$ and 
for the noisy cases, $e_n$’s are independent and identically distributed Gaussian random variables with 
mean 0 and standard deviation 0.05.

\noindent\textbf{Testing data}: $\{(x'_n, y'_n)\}_{n = 1}^{N'} \subset [a, b] \times \mathbb{R}$, where  $N' := 1,000$,
 $x'_n$'s are equally spaced on $[a, b]$ and
 given $x'_n$, the corresponding $y'_n$ is computed by
$y'_n = f(x'_n)$.

\noindent\textbf{Validation data}: Due to the high frequency of the target function, we prefer to have as much training data as we can. Therefore, instead of leaving out some training data as validation data, we randomly copy $20\%$ of the training data and add Gaussian noise with mean $0$ and  standard deviation $0.01$ to the corresponding $y$ values. The validation data remains the same for both the multi-grade model and the classical single-grade model.

Given predictions $\hat{y}_n$ for $y_n$, the mean squared error for training is defined by
\begin{equation}
\mathrm{mse\ (train)}:= \frac{1}{N}\sum_{n = 1}^N(\hat{y}_n - y_n)^2
\end{equation}
and the relative squared error for training is defined by
\begin{equation}
\mathrm{rse\ (train)}:= \frac{\sum_{n = 1}^N(\hat{y}_n - y_n)^2}{\sum_{n = 1}^Ny_n^2}.
\end{equation}
Likewise, suppose that $\hat{y}_n$ is an approximation for $y'_n$ and we define the mean squared error and the relative squared error on the testing data respectively by
\begin{equation}
\mathrm{mse\ (test)}:= \frac{1}{N'}\sum_{n = 1}^{N'}(\hat{y}_n - y'_n)^2
\end{equation}
and
\begin{equation}
\mathrm{rse\ (test)}:= \frac{\sum_{n = 1}^{N'}(\hat{y}_n - y'_n)^2}{\sum_{n = 1}^{N'}(y'_n)^2}.
\end{equation}

For the three examples, both the noise free and noisy cases, each grade of the multi-grade model is trained with the Adam optimizer and with batch size 32 and a
learning rate decay $1.0 \times 10^{-2}$. For the single-grade learning model, we use three different epochs in training for different examples, aiming at obtaining the best outcome for each individual example. The epochs will be specified in each example. The training time reported in this section is measured in seconds.

\bigskip

\noindent\textbf{Example 1:} 
%
We consider approximating the oscillatory  function
\begin{equation*}
f(x) = \mathrm{sin}(100x),\ x \in [a,b]:=[0, 1].
\end{equation*}
%
For this example, we use the following three grade model:
$$
\begin{aligned}
&\mbox{Grade 1}:\ [1] \to [256] \to [256] \to [1],\\
&\mbox{Grade 2}:\ [1] \to [256]_{\mbox{F}} \to [256]_{\mbox{F}} \to [128] \to [128] \to [64] \to [1],\\
&\mbox{Grade 3}:\ [1] \to [256]_{\mbox{F}} \to [256]_{\mbox{F}} \to [128]_{\mbox{F}} \to [128]_{\mbox{F}} \to [64]_{\mbox{F}} \to [64] \to [32] \to [32] \to [1].
\end{aligned}
$$
Here, $[n]$ indicates a fully connected layer with $n$ neurons, and $[n]_{\mbox{F}}$ indicates $[n]$ with all parameters fixed during training. We choose the activation function for hidden layers in the first grade to be the $\mathrm{sin}$ function and that for the remaining hidden layers to be the ReLU function. No activation function is applied to the last layer of each grade.

\begin{figure}[htp]
\centering
\includegraphics[width=.4\textwidth]{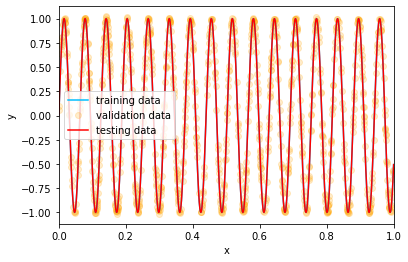}
\caption{Example 1 (noise free): The training, testing and validation data.}
\label{figure_exp1_data}
\end{figure}

The corresponding classic single-grade architecture is given by
$$
[1] \to [256] \to [256] \to [128] \to [128] \to [64] \to [64] \to [32] \to [32] \to [1],
$$
in which the activation function for the first two hidden layers is the $\mathrm{sin}$ function and that for the remaining hidden layers is the ReLU function, the same as in the multi-grade model. No activation is applied to the last layer. The classical single-grade network is trained with Adam: learning rate $0.01$ with decay $1.0\times 10^{-2}$
and batch size 32 for 550 epochs.

\begin{table}[ht]
\caption{Example 1 (noise free): Training time and accuracy for each grade.}
\centering
\begin{tabular}{c|c|c|c|c|c}
\hline
grade&learning rate&epochs&training time &mse (train)&rse (train)\\
\hline
1&$0.1$&$5$&$4.4092$&$1.7961 \times 10^{-2}$&$3.5771 \times 10^{-2}$\\
2&$0.01$&$20$&$14.5209$&$1.1362 \times 10^{-4}$&$2.2629 \times 10^{-4}$\\
3&$0.01$&$40$&$30.0601$&$2.0203 \times 10^{-5}$&$4.0237 \times 10^{-5}$\\
\hline
\end{tabular}
\label{table_exp1_grades}
\end{table}

The training, testing and the validation data for the noise free case of Example 1 are plotted in Figure \ref{figure_exp1_data}. Numerical results for this case are listed in Tables \ref{table_exp1_grades} and \ref{table_exp1}. The corresponding figures are plotted in Figures \ref{figure_exp1_train} and \ref{figure_exp1_test}.

\begin{table}[ht]
\caption{Example 1 (noise free):  multi-grade vs  single-grade.}
\centering
\begin{tabular}{c|c|c|c|c|c}
\hline
method&training time &mse (train)&rse (train)&mse (test)&rse (test)\\
\hline
multi-grade&48.9902&$2.0203 \times 10^{-5}$&$4.0237 \times 10^{-5}$&$2.1450 \times 10^{-5}$&$4.2746 \times 10^{-5}$\\
single-grade&485.7447&$3.1399 \times 10^{-1}$&$6.2534 \times 10^{-1}$&$3.1374 \times 10^{-1}$&$6.2522 \times 10^{-1}$\\
\hline
\end{tabular}
\label{table_exp1}
\end{table}






\begin{figure}
\centering
  \begin{tabular}{c c }
    \includegraphics[width=.4\textwidth]{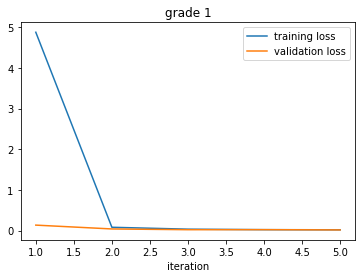} &
    \includegraphics[width=.4\textwidth]{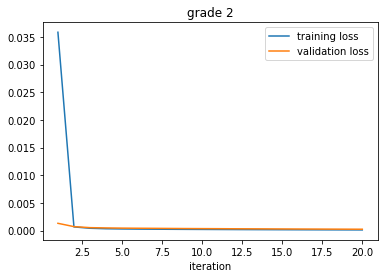}\\
    \includegraphics[width=.4\textwidth]{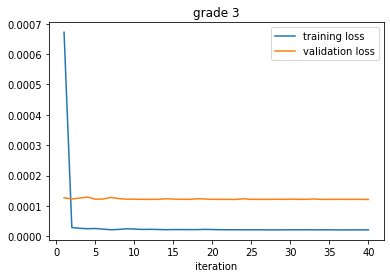} &
    \includegraphics[width=.4\textwidth]{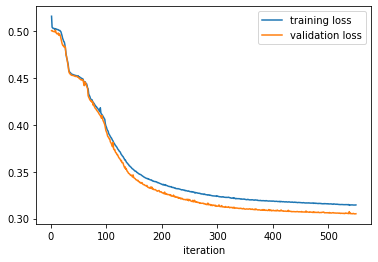}
  \end{tabular}

\caption{Example 1 (noise free): Loss curves during training: (a)-(c) multi-grade; (d) single-grade.}
\label{figure_exp1_train}
\end{figure}

\begin{figure}[htp]
\centering
  \begin{tabular}{c c}
        \includegraphics[width=.4\textwidth]{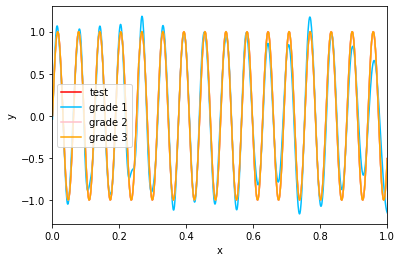} &\includegraphics[width=.4\textwidth]{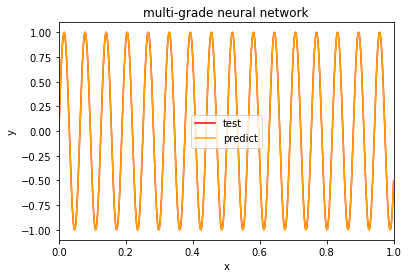}\\

  \end{tabular}

  \begin{tabular}{c}
  \includegraphics[width=.4\textwidth]{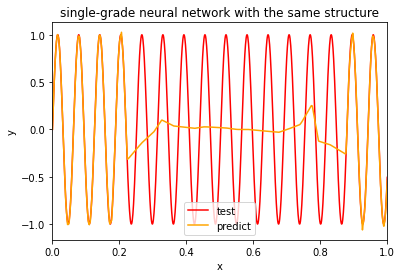}
  \end{tabular}

\caption{Example 1 (noise free) - Predictions on the testing set: (a) after each grade of multi-grade; (b) multi-grade; (c) single-grade.}
\label{figure_exp1_test}
\end{figure}

The training, testing and the validation data for the noisy case of Example 1 are plotted in Figure \ref{figure_exp2_data}.
Numerical results for this case are listed in Tables \ref{table_exp2_grades} and \ref{table_exp2}, and the corresponding figures are plotted in Figures \ref{figure_exp2_train} and \ref{figure_exp2_test}.

\begin{figure}[htp]
\centering
\includegraphics[width=.4\textwidth]{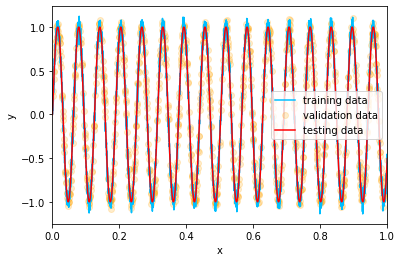}
\caption{Example 1 (noisy case): The training, testing and the validation data.}
\label{figure_exp2_data}
\end{figure}

\begin{table}[ht]
\caption{Example 1 (noisy case): Training time and accuracy for each grade.}
\centering
\begin{tabular}{c|c|c|c|c|c}
\hline
grade&learning rate&epochs&training time &mse (train)&rse (train)\\
\hline
1&$0.1$&$5$&$4.3172$&$4.3381 \times 10^{-2}$&$8.6013 \times 10^{-2}$\\
2&$0.01$&$20$&$13.9507$&$2.5290 \times 10^{-3}$&$5.0143 \times 10^{-2}$\\
3&$0.01$&$40$&$29.8275$&$2.4058 \times 10^{-3}$&$4.7701 \times 10^{-2}$\\
\hline
\end{tabular}
\label{table_exp2_grades}
\end{table}

\begin{table}[ht]
\caption{Example 1 (noisy case):  multi-grade vs  single-grade.}
\centering
\begin{tabular}{c|c|c|c|c|c}
\hline
method&training time &mse (train)&rse (train)&mse (test)&rse (test)\\
\hline
multi-grade&48.0953&$2.4058 \times 10^{-3}$&$4.7701 \times 10^{-3}$&$1.1476 \times 10^{-4}$&$2.2869 \times 10^{-4}$\\
single-grade&474.8579&$3.1994 \times 10^{-1}$&$6.3436 \times 10^{-1}$&$3.1720 \times 10^{-1}$&$6.3213 \times 10^{-1}$\\
\hline
\end{tabular}
\label{table_exp2}
\end{table}

\begin{figure}[htp]
\centering
  \begin{tabular}{cc}
    \includegraphics[width=.4\textwidth]{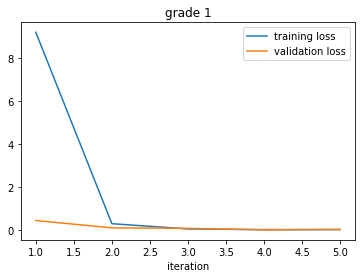} &
     \includegraphics[width=.4\textwidth]{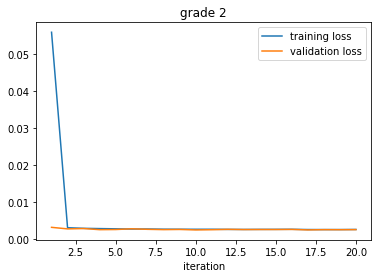}\\     
   \includegraphics[width=.4\textwidth]{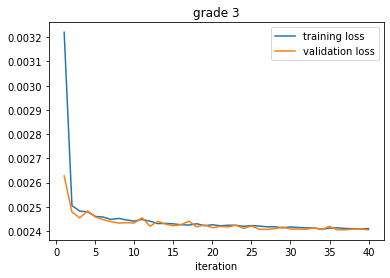} &
    \includegraphics[width=.4\textwidth]{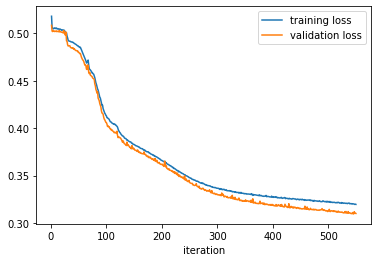}  
  \end{tabular}

\caption{Example 1 (noisy case): Loss curves during training. (a)-(c) multi-grade; (d) single-grade.}
\label{figure_exp2_train}
\end{figure}

\begin{figure}[htp]
\centering
  \begin{tabular}{cc}
    \includegraphics[width=.4\textwidth]{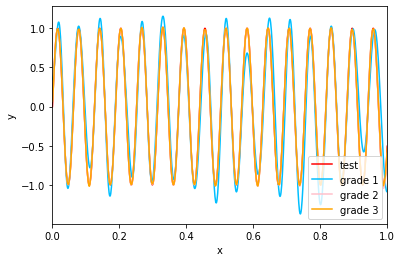}&
    \includegraphics[width=.4\textwidth]{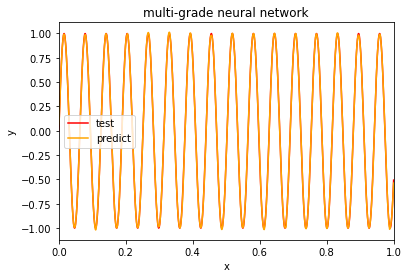}\\
    
  \end{tabular}

  \begin{tabular}{c}
  \includegraphics[width=.4\textwidth]{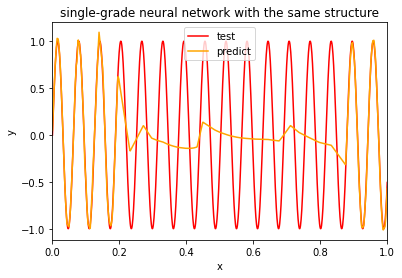}
  \end{tabular}

\caption{Example 1 (noisy case): Predictions on the testing set. (a) after each grade of multi-grade; (b) multi-grade; (c) single-grade.}
\label{figure_exp2_test}
\end{figure}

\bigskip

\noindent\textbf{Example 2:} We consider approximating the oscillatory function with a variable magnitude
\begin{equation*}
f(x) := x\mathrm{sin}(100x),\ \ x \in [a,b]:= [0, 1].
\end{equation*}

\begin{figure}[htp]
\centering
\includegraphics[width=.45\textwidth]{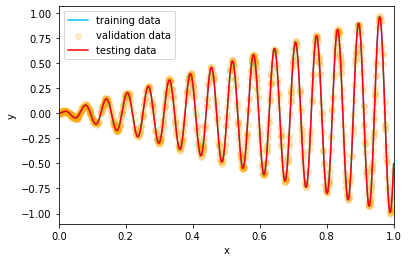}
\caption{Example 2 (noise free): The training, testing and validation data.}
\label{figure_exp3_data}
\end{figure}

For this example, we use a three grade model:
$$
\begin{aligned}
&\mbox{Grade 1}:\ [1] \to [256] \to [256] \to [1],\\
&\mbox{Grade 2}:\ [1] \to [256]_{\mbox{F}} \to [256]_{\mbox{F}} \to [128] \to [64] \to [64] \to [1],\\
&\mbox{Grade 3}:\ [1] \to [256]_{\mbox{F}} \to [256]_{\mbox{F}} \to [128]_{\mbox{F}} \to [64]_{\mbox{F}} \to [64]_{\mbox{F}} \to [64] \to [32] \to [32] \to [1].
\end{aligned}
$$
The activation function for hidden layers in the first grade is the $\sin$ function and the activation function for the rest of hidden layers is the ReLU function. No activation function is applied to the last layer of each grade. The corresponding classic architecture is given by
$$
[1] \to [256] \to [256] \to [128] \to [64] \to [64] \to [64] \to [32] \to [32] \to [1],
$$
in which the activation function for the first two hidden layers is the $\mathrm{sin}$ function and the activation for the rest of hidden layers is the ReLU function. No activation is applied to the last layer. The single-grade network is trained with Adam: learning rate $0.01$ with decay $1.0 \times 10^{-2}$ and batch size $32$ for $800$ epochs.

For the noise free case of Example 2, the
training, testing and the validation data are plotted in Figure \ref{figure_exp3_data}.
Numerical results for this case are listed in Tables \ref{table_exp3_grades} and \ref{table_exp3}, and the corresponding figures are plotted in Figures \ref{figure_exp3_train} and \ref{figure_exp3_test}.

\begin{table}[ht]
\caption{Example 2 (noise free): Training time and accuracy for each grade.}
\centering
\begin{tabular}{c|c|c|c|c|c}
\hline
grade&learning rate&epochs&training time &mse (train)&rse (train)\\
\hline
1&$0.1$&$5$&$4.7529$&$3.1325 \times 10^{-2}$&$1.8554 \times 10^{-1}$\\
2&$0.01$&$20$&$15.1325$&$1.1796 \times 10^{-5}$&$6.9869 \times 10^{-5}$\\
3&$0.01$&$40$&$39.5186$&$3.8516 \times 10^{-6}$&$2.2813 \times 10^{-5}$\\
\hline
\end{tabular}
\label{table_exp3_grades}
\end{table}

\begin{table}[ht]
\caption{Example 2 (noise free):  multi-grade vs  single-grade.}
\centering
\begin{tabular}{c|c|c|c|c|c}
\hline
method&training time &mse (train)&rse (train)&mse (test)&rse (test)\\
\hline
multi-grade&59.4039&$3.8516 \times 10^{-6}$&$2.2813 \times 10^{-5}$&$3.9175 \times 10^{-6}$&$2.3210 \times 10^{-5}$\\
single-grade&734.3698&$7.7438 \times 10^{-3}$&$4.5867 \times 10^{-2}$&$7.7376 \times 10^{-3}$&$4.5842 \times 10^{-2}$\\
\hline
\end{tabular}
\label{table_exp3}
\end{table}

\begin{figure}[htp]
\centering
\begin{tabular}{cc}
    \includegraphics[width=.4\textwidth]{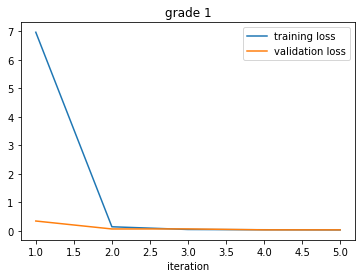} &
    \includegraphics[width=.4\textwidth]{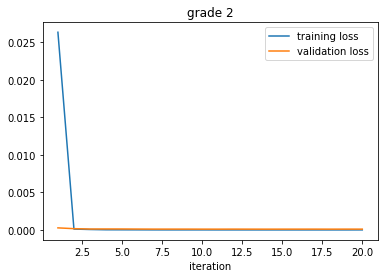}\\
    \includegraphics[width=.4\textwidth]{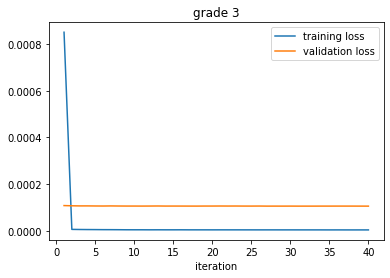} &
    \includegraphics[width=.4\textwidth]{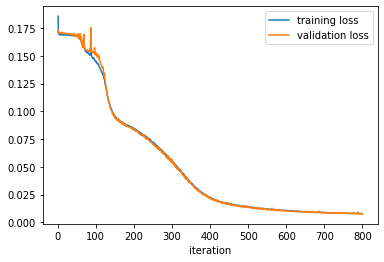}
  \end{tabular}

\caption{Example 2 (noise free) - Loss curves during training: (a)-(c) multi-grade; (d) single-grade.}
\label{figure_exp3_train}
\end{figure}

\begin{figure}[htp]
\centering
\begin{tabular}{cc}
    \includegraphics[width=.4\textwidth]{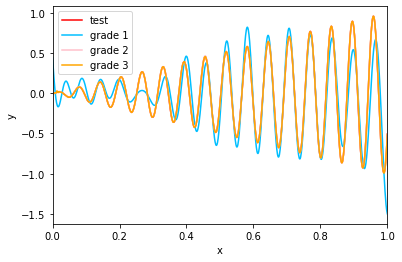} &
    \includegraphics[width=.4\textwidth]{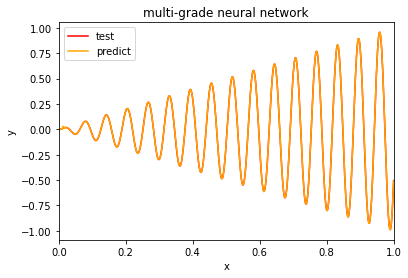}\\
  \end{tabular}
\begin{tabular}{c}
    \includegraphics[width=.4\textwidth]{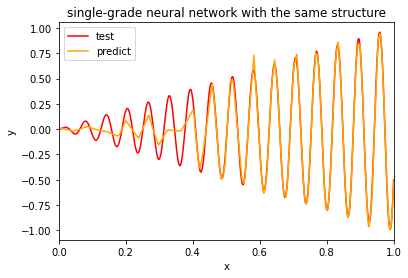}
\end{tabular}  

\caption{Example 2 (noise free) - Predictions on the testing set: (a) after each grade of multi-grade; (b) multi-grade; (c) single-grade.}
\label{figure_exp3_test}
\end{figure}

For the noisy case of Example 2, the training, testing and the validation data are plotted in Figure \ref{figure_exp4_data}. Numerical results for this case are listed in Tables \ref{table_exp4_grades} and \ref{table_exp4}, and the corresponding figures are plotted in Figures \ref{figure_exp4_train} and \ref{figure_exp4_test}.

\begin{figure}[htp]
\centering
\includegraphics[width=.45\textwidth]{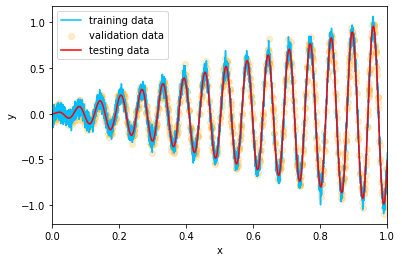}
\caption{Example 2 (noisy case): The training, testing and validation data.}
\label{figure_exp4_data}
\end{figure}

\begin{table}[ht]
\caption{Example 2 (noisy case): Training time and accuracy for each grade.}
\centering
\begin{tabular}{c|c|c|c|c|c}
\hline
grade&learning rate&epochs&training time &mse (train)&rse (train)\\
\hline
1&$0.1$&$5$&$2.9456$&$9.4412 \times 10^{-3}$&$5.5099 \times 10^{-2}$\\
2&$0.01$&$20$&$10.7950$&$2.5191 \times 10^{-3}$&$1.4701 \times 10^{-2}$\\
3&$0.01$&$40$&$24.0130$&$2.4562 \times 10^{-3}$&$1.4334 \times 10^{-2}$\\
\hline
\end{tabular}
\label{table_exp4_grades}
\end{table}

\begin{table}[ht]
\caption{Example 2 (noisy case):  multi-grade vs  single-grade.}
\centering
\begin{tabular}{c|c|c|c|c|c}
\hline
method&training time &mse (train)&rse (train)&mse (test)&rse (test)\\
\hline
multi-grade&37.7535&$2.4562 \times 10^{-3}$&$1.4334 \times 10^{-2}$&$9.8123 \times 10^{-5}$&$5.8133 \times 10^{-4}$\\
single-grade&684.0808&$1.3066 \times 10^{-2}$&$7.6253 \times 10^{-2}$&$1.0483 \times 10^{-2}$&$6.2108 \times 10^{-2}$\\
\hline
\end{tabular}
\label{table_exp4}
\end{table}

\begin{figure}[htp]
\centering

  \begin{tabular}{cc}
    \includegraphics[width=.4\textwidth]{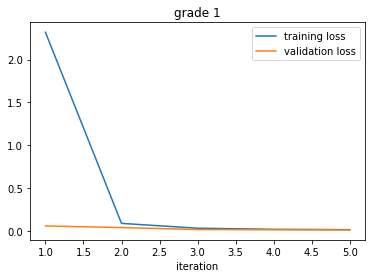} &
    \includegraphics[width=.4\textwidth]{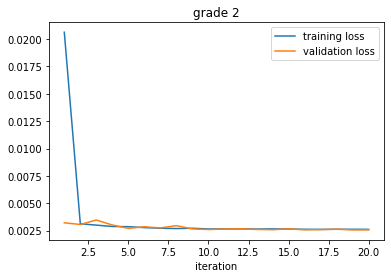}\\
    \includegraphics[width=.4\textwidth]{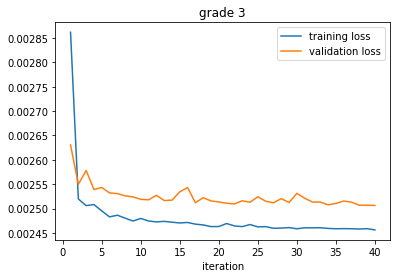}&
    \includegraphics[width=.4\textwidth]{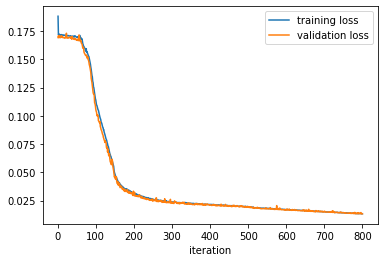}
  \end{tabular} 
\caption{Example 2 - Loss curves during training: (a)-(c) multi-grade; (d) single-grade.}
\label{figure_exp4_train}
\end{figure}

\begin{figure}[htp]
\centering
  \begin{tabular}{cc}
    \includegraphics[width=.4\textwidth]{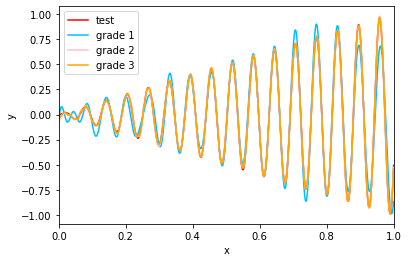} &
    \includegraphics[width=.4\textwidth]{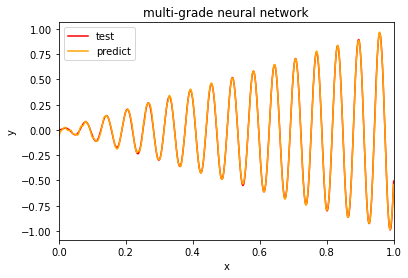}
  \end{tabular}

\begin{tabular}{c}
  \includegraphics[width=.4\textwidth]{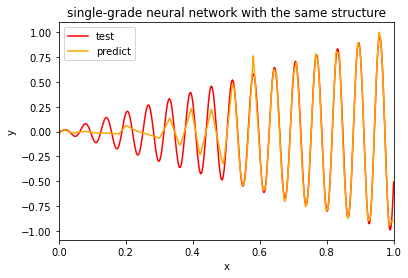}
\end{tabular}
  
\caption{Example 2 - Predictions on the testing set: (a) after each grade of multi-grade; (b) multi-grade; (c) single-grade.}
\label{figure_exp4_test}
\end{figure}

\bigskip

\noindent{\bf Example 3:} In this example, we consider a non-differentiable function
\begin{equation}
f(x) := (x + 1) (f_4 \circ f_3 \circ f_2 \circ f_1)(x),\ x \in [a,b]:= [-1, 1],
\end{equation}
where $(f \circ g)(x)$ means the composition $f(g(x))$ assuming the image of $g$ lies in the domain of $f$, and $f_k$'s are given by
\begin{equation*}
\begin{aligned}
&f_1(x) := \vert\mathrm{cos}(\pi(x - 0.3)) - 0.7\vert,\\
&f_2(x) := \vert\mathrm{cos}(2\pi(x - 0.5)) - 0.5\vert,\\
&f_3(x) := -\vert x - 1.3\vert + 1.3,\\
&f_4(x) := -\vert x - 0.9\vert + 0.9.\\
\end{aligned}
\end{equation*}

\begin{figure}[htp]
\centering
\includegraphics[width=.45\textwidth]{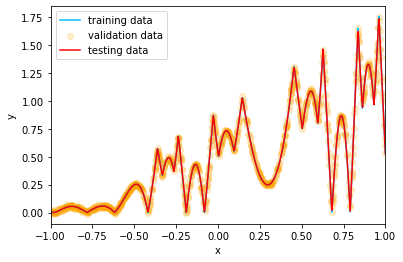}
\caption{Example 3 (noise free): The training, testing and the validation data. }
\label{figure_exp5_data}
\end{figure}

For this example, we use the following three grades:
$$
\begin{aligned}
&\mbox{Grade 1}:\ [1] \to [128] \to [128] \to [1],\\
&\mbox{Grade 2}:\ [1] \to [128]_{\mbox{F}} \to [128]_{\mbox{F}} \to [64] \to [64] \to [64] \to [1],\\
&\mbox{Grade 4}:\ [1] \to [128]_{\mbox{F}} \to [128]_{\mbox{F}} \to [64]_{\mbox{F}} \to [64]_{\mbox{F}} \to [64]_{\mbox{F}} \to [32] \to [32] \to [32] \to [1].
\end{aligned}
$$
The activation function for hidden layers in the first grade is the $\sin$ function and the activation function for the rest of hidden layers is the ReLU function. No activation function is applied to the last layer of each grade. The corresponding classic architecture is given by
$$
[1] \to [128] \to [128] \to [64] \to [64] \to [64] \to [32] \to [32] \to [32] \to [1],
$$
in which the activation function for the first two hidden layers is the $\sin$ function and the activation for the rest of hidden layers is the ReLU function. No activation is applied to the last layer. The single-grade network is trained with Adam: learning rate $0.01$ with decay $1.0 \times 10^{-2}$ and batch size $32$ for $500$ epochs.

For the noise free case of Example 3, the training, testing and the validation data are plotted in Figure \ref{figure_exp5_data}.
Numerical results for this case are listed in Tables \ref{table_exp5_grades} and \ref{table_exp5}. The corresponding figures are plotted in Figures \ref{figure_exp5_train} and \ref{figure_exp5_test}.

\begin{table}[ht]
\caption{Example 3 (noise free): Training time and accuracy for each grade.}
\centering
\begin{tabular}{c|c|c|c|c|c}
\hline
grade&learning rate&epochs&training time &mse (train)&rse (train)\\
\hline
1&$0.1$&$5$&$3.9165$&$2.9048 \times 10^{-3}$&$6.9275 \times 10^{-3}$\\
2&$0.01$&$10$&$6.7180$&$8.8397 \times 10^{-6}$&$2.1081 \times 10^{-5}$\\
3&$0.01$&$40$&$26.5470$&$2.8498 \times 10^{-6}$&$6.7963 \times 10^{-6}$\\
\hline
\end{tabular}
\label{table_exp5_grades}
\end{table}

\begin{table}[ht]
\caption{Example 3 (noise free):  multi-grade vs  single-grade.}
\centering
\begin{tabular}{c|c|c|c|c|c}
\hline
method&training time &mse (train)&rse (train)&mse (test)&rse (test)\\
\hline
multi-grade&37.1816&$2.8498 \times 10^{-6}$&$6.7963 \times 10^{-6}$&$3.0007 \times 10^{-6}$&$7.1605 \times 10^{-6}$\\
single-grade&377.6011&$1.3259 \times 10^{-2}$&$3.1620 \times 10^{-2}$&$1.3243 \times 10^{-2}$&$3.1602 \times 10^{-2}$\\
\hline
\end{tabular}
\label{table_exp5}
\end{table}

\begin{figure}[htp]
\centering

  \begin{tabular}{cc}
    \includegraphics[width=.4\textwidth]{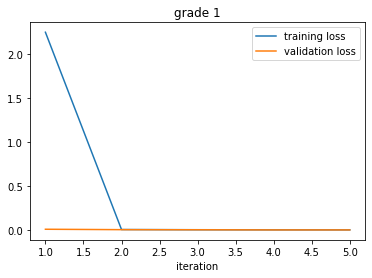} &
     \includegraphics[width=.4\textwidth]{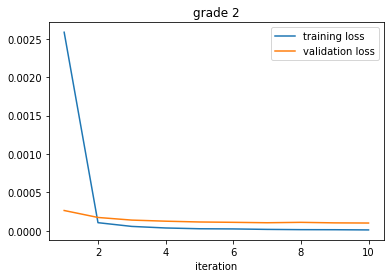}\\
      \includegraphics[width=.4\textwidth]{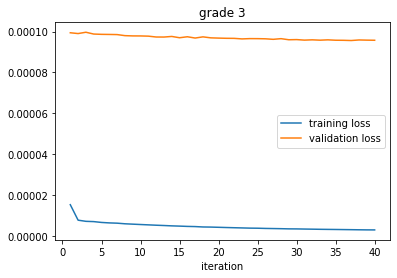}&
    \includegraphics[width=.4\textwidth]{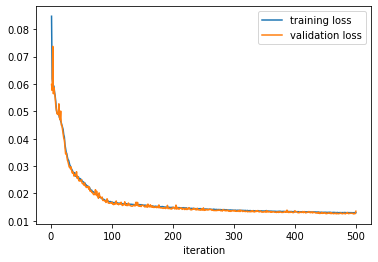}
  \end{tabular}

\caption{Example 3 (noise free) - Loss curves during training: (a)-(c) multi-grade; (d) single-grade.}
\label{figure_exp5_train}
\end{figure}

\begin{figure}[htp]
\centering
  \begin{tabular}{cc}
    \includegraphics[width=.4\textwidth]{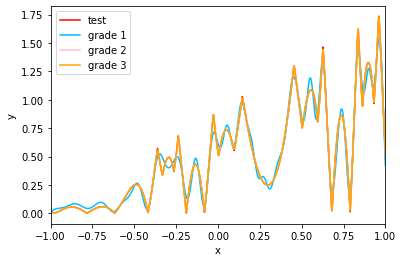} &
    \includegraphics[width=.4\textwidth]{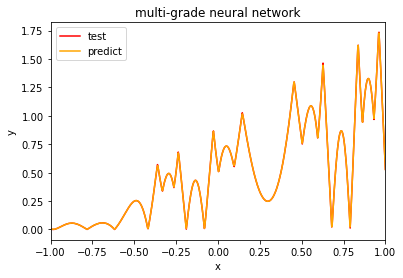}
  \end{tabular}

  \begin{tabular}{c}
    \includegraphics[width=.4\textwidth]{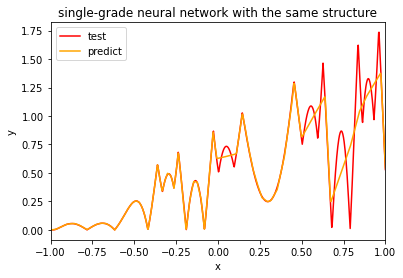}
  \end{tabular}

\caption{Example 3 (noise free) - Predictions on the testing set: (a) after each grade of multi-grade; (b) multi-grade; (c) single-gade.}
\label{figure_exp5_test}
\end{figure}

The training, testing and the validation data for the noisy case of Example 3 are plotted in Figure \ref{figure_exp6_data}. Numerical results for this case are listed in Tables \ref{table_exp6_grades} and \ref{table_exp6}, and the corresponding figures are plotted in Figures \ref{figure_exp6_train} and \ref{figure_exp6_test}.

\begin{figure}[htp]
\centering
\includegraphics[width=.45\textwidth]{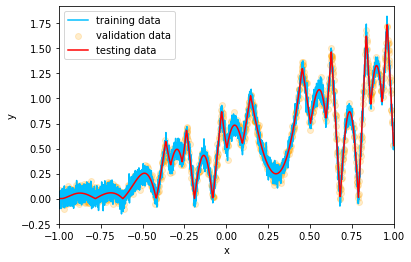}
\caption{Example 3 (noisy case): The training, testing and the validation data. }
\label{figure_exp6_data}
\end{figure}

\begin{table}[ht]
\caption{Example 3 (noisy case): Training time and accuracy for each grade.}
\centering
\begin{tabular}{c|c|c|c|c|c}
\hline
grade&learning rate&epochs&training time &mse (train)&rse (train)\\
\hline
1&$0.1$&$5$&$4.0153$&$2.8773 \times 10^{-2}$&$6.8177 \times 10^{-2}$\\
2&$0.01$&$10$&$6.8856$&$2.6177 \times 10^{-3}$&$6.2027 \times 10^{-3}$\\
3&$0.01$&$25$&$16.3114$&$2.5108 \times 10^{-3}$&$5.9495 \times 10^{-3}$\\
\hline
\end{tabular}
\label{table_exp6_grades}
\end{table}

\begin{table}[ht]
\caption{Example 3 (noisy case):  multi-grade vs  single-grade.}
\centering
\begin{tabular}{c|c|c|c|c|c}
\hline
method&training time &mse (train)&rse (train)&mse (test)&rse (test)\\
\hline
multi-grade&27.2122&$2.5108 \times 10^{-3}$&$5.9495 \times 10^{-3}$&$1.5818 \times 10^{-4}$&$3.7746 \times 10^{-4}$\\
single-grade&406.6714&$1.8728 \times 10^{-2}$&$4.4376 \times 10^{-2}$&$1.6499 \times 10^{-2}$&$3.9372 \times 10^{-2}$\\
\hline
\end{tabular}
\label{table_exp6}
\end{table}

\begin{figure}[htp]
\centering

  \begin{tabular}{cc}
    \includegraphics[width=.4\textwidth]{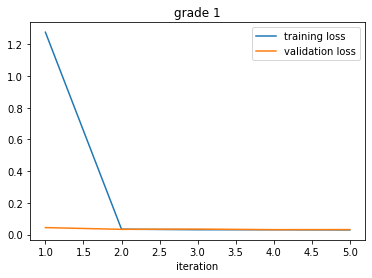} &
     \includegraphics[width=.4\textwidth]{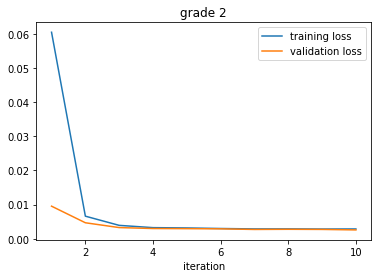}\\
     \includegraphics[width=.4\textwidth]{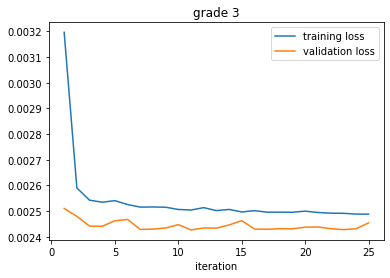}&
     \includegraphics[width=.4\textwidth]{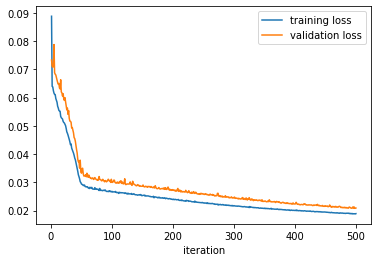}
  \end{tabular}

\caption{Example 3 (noisy case) - Loss curves during training: (a)-(c) multi-grade; (d) single-grade.}
\label{figure_exp6_train}
\end{figure}

\begin{figure}[htp]
\centering
  \begin{tabular}{cc}
    \includegraphics[width=.4\textwidth]{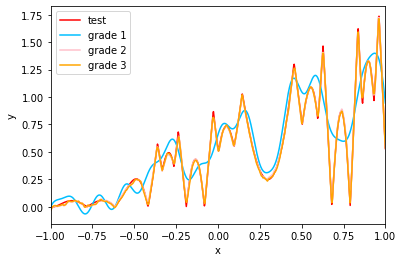}&
    \includegraphics[width=.4\textwidth]{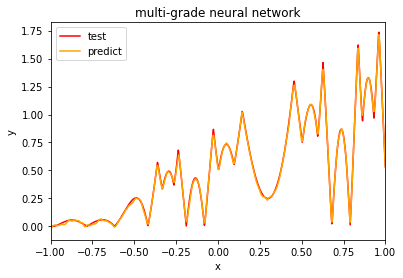}\\
    
  \end{tabular}

  \begin{tabular}{c}
  \includegraphics[width=.4\textwidth]{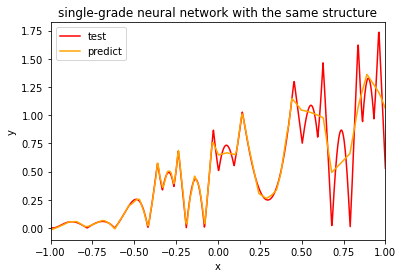}
  \end{tabular}
\caption{Example 3 (noisy case) - Predictions on the testing set: (a) after each grade of multi-grade; (b) multi-grade; (c) single-grade.}
\label{figure_exp6_test}
\end{figure}

We observe from the numerical results presented in this section that for all examples (for both noise free and noisy cases) the multi-grade model outperforms significantly the single-grade model in terms of training time, training accuracy and prediction accuracy in all counts. Moreover, multi-grade models are much easier to train than their corresponding single-grade models.

\section{Conclusive Remarks}
To address computational challenges in learning deep neural networks, we have proposed multi-grade learning models, inspired by the human education process which arranges learning in grades. Unlike the deep neural network learned from the standard single-grade model, the neural network learned from a multi-grade model is the superposition of the  neural networks, in a {\it stair-shape}, each of which is learned from one grade of the learning.  We have proved that for function approximation if the neural network generated by a new grade is nontrivial, the optimal error of the grade is strictly reduced from the optimal error of the previous grade. Three proof-of-concept numerical examples presented in the paper demonstrate that the proposed  
the multi-grade deep learning model outperforms significantly the single-grade deep learning model in all aspects, including training time, training accuracy and prediction accuracy. Our computational experiences show that the multi-grade model is much easier to train than the corresponding
single-grade model. We conclude that the proposed  multi-grade model is much more robust than the traditional single-grade model. 
\bigskip

\noindent {\bf Acknowledgement:}
The author is indebted to graduate student Mr. Lei Huang for his assistance in preparation of the numerical examples presented in section 6 and to Dr. Tingting Wu for helpful discussion of issues related to approximation of the oscillatory function. 
The author of this paper is supported in part by US National Science Foundation under grants DMS-1912958  and DMS-2208386, and by the US National Institutes of Health under grant R21CA263876.

\end{document}